\documentclass[10pt,twocolumn,letterpaper]{article}

\usepackage{cvpr}              
\usepackage{algorithm}
\usepackage{algorithmic}
\usepackage{booktabs}
\usepackage{multirow}
\usepackage{tabularx, booktabs} 
\usepackage{threeparttable}
\usepackage{diagbox}  
\usepackage{amsmath, amssymb}
\usepackage{amsthm, mathtools}
\newtheorem{definition}{Definition}[section]
\newtheorem{lemma}{Lemma}[section]

\newtheorem{theorem}{Theorem}[section]

\newtheorem{corollary}{Corollary}[section]
\newtheorem{remark}{Remark}[section] 
\usepackage{subcaption}
\usepackage[percent]{overpic}
\usepackage{pifont}
\usepackage{xcolor}  
\usepackage{thmtools} 
\usepackage{adjustbox}
\usepackage[table]{xcolor}
\definecolor{dpssu}{HTML}{FFF9C4}  
\definecolor{midgray}{HTML}{E6E6E6}  
\definecolor{EC805A}{HTML}{EC805A}
\definecolor{4575b4}{HTML}{4575b4}
\usepackage{pifont}
\usepackage{contour}            
\definecolor{staryellow}{RGB}{255,205,0}
\contourlength{0.12ex}          

\newcolumntype{Y}{>{\centering\arraybackslash}X}
\newcolumntype{L}{>{\raggedright\arraybackslash}X}
\definecolor{MorandiBlue}{HTML}{2A55A4}
\definecolor{vpdr}{HTML}{D7F6FF}  
\usepackage{diagbox}

\makeatletter
\g@addto@macro\normalsize{%
  \setlength\abovedisplayskip{4pt}%
  \setlength\belowdisplayskip{4pt}%
  \setlength\abovedisplayshortskip{0pt plus 2pt}%
  \setlength\belowdisplayshortskip{4pt plus 2pt minus 2pt}%
  \setlength\jot{2pt} 
}
\makeatother
 
\usepackage{tikz}
\usepackage[utf8]{inputenc}
 
\newcommand{\blackcircnum}[1]{%
  \tikz[baseline=-0.75ex]{
    \node[
      shape=circle,
      draw=black,
      fill=black,
      text=white,
      inner sep=0.5pt,
      font=\scriptsize\bfseries,
      minimum size=0.8em
    ] (char) {#1};
  }%
} 

\usepackage{marvosym} 

\definecolor{cvprblue}{rgb}{0.21,0.49,0.74}
\usepackage[pagebackref,breaklinks,colorlinks,allcolors=cvprblue]{hyperref}

\def\paperID{5310} 
\def\confName{CVPR}
\def\confYear{2026}

\title{Taming Noise-Induced Prototype Degradation for Privacy-Preserving Personalized Federated Fine-Tuning}

\author{
Yuhua Wang$^{1}$,
Qinnan Zhang$^{1}$\textsuperscript{*},
Xiaodong Li$^{2}$,
Huan Zhang$^{1}$,
Yifan Sun$^{2}$\textsuperscript{*},
Wangjie Qiu$^{1}$,\\
Hainan Zhang$^{1}$,
Yongxin Tong$^{3}$,
Zhiming Zheng$^{1}$\\  
$^{1}$School of Artificial Intelligence, Beihang University\\
$^{2}$School of Statistics, Renmin University of China\\
$^{3}$School of Computer Science and Engineering, Beihang University \\
\texttt{\{yuhuawang, zhangqn\}@buaa.edu.cn; 
sunyifan@ruc.edu.cn}  
}

\begin{document}
\maketitle

\begingroup
\renewcommand\thefootnote{\fnsymbol{footnote}}
\footnotetext[1]{Corresponding Authors. 
This work was supported by Beijing Advanced Innovation Center for Future
Blockchain and Privacy Computing. 
}
\endgroup

\begin{abstract} 
Prototype-based Personalized Federated Learning (ProtoPFL) enables efficient multi-domain adaptation by communicating compact class prototypes, but directly sharing them poses privacy risks.
A common defense involves per-example $\ell_2$ clipping before prototype computation to bound  sensitivity, followed by isotropic Gaussian noise to enforce Local Differential Privacy (LDP). However, Isotropic Gaussian Prototype Perturbation (IGPP) typically over-perturbs discriminative dimensions and struggles to balance the clipping threshold with representation fidelity.
In this paper, we propose \textbf{VPDR}, a client-side privacy plug-in that seamlessly integrates into existing ProtoPFLs. 
Motivated by the observation that dimension-wise class variance reflects discriminability, we introduce \textbf{Variance-adaptive Prototype Perturbation (VPP)}, which allocates less noise to discriminative subspaces, preserving semantic separability while ensuring privacy.
We further develop \textbf{Distillation-guided Clipping Regularization (DCR)}, which enables feature norms to adaptively concentrate near the predefined clipping threshold while maintaining prediction consistency.
Theoretical analysis shows that our groupwise mechanism provides privacy guarantees no weaker than the isotropic baseline under the same privacy constraints.
Extensive experiments on multi-domain benchmarks demonstrate that VPDR achieves a superior privacy-utility trade-off, outperforming IGPP in personalized federated fine-tuning without sacrificing robustness against realistic attacks.  
\end{abstract}

\begin{figure}[t]
  \centering
  \begin{subfigure}[b]{1\linewidth}
    \includegraphics[width=\linewidth]{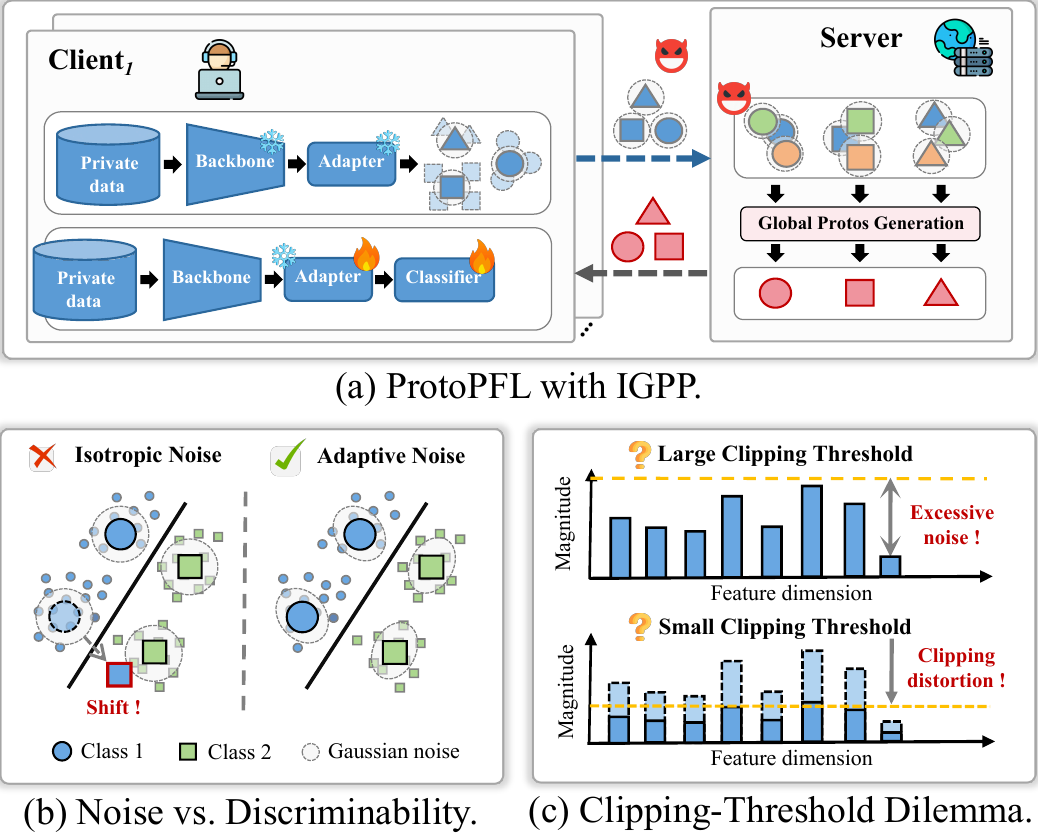}
  \end{subfigure} 
  \vspace{-6mm}
  \caption{\textbf{Motivation illustration}. In (b), the dashed circle shows uniform noise and the dashed ellipse our adaptive noise. Isotropic noise shifts the blue Class 1 prototype into the red Class 2 region. In (c), a generous clipping bound leaves feature norms almost unchanged but forces a large noise scale, while an aggressive bound shrinks many features, causing severe information loss.}
  \label{fig:issues}
\end{figure}

\section{Introduction}
Scaling laws for foundation models have been established on massive public datasets~\cite{alabdulmohsin2022revisiting}.
As general pre-training has largely consumed available public data~\cite{villalobos2024position}, leveraging private data becomes essential to overcome the emerging data-scarcity bottleneck~\cite{muennighoff2023scaling}.
Federated Learning (FL)~\cite{mcmahan2017communication,bai2024diprompt,lee2024fedsol,shi2025fedawa} offers a distributed paradigm for collaboratively fine-tuning pre-trained models without sharing raw data~\cite{Ye2024OpenFedLLMTL,Zheng2024SafelyLW}. Under domain skew, however, a single global model often underperforms or even fails to converge across clients~\cite{Wan2024FederatedGL,chen2024fair,wang2025federated}. 
To solve this, Personalized FL (PFL)~\cite{t2020pfedme,fallah2020personalized,ditto2021,xie2024perada} adapts shared models with client-specific components~\cite{t2020pfedme,ma2022,zhang2023fedala,tamirisa2024fedselect,yang2024fedas}.
In PFL, prototype-based methods~\cite{FedProto2022, fedpcl2022, fplhuang2023, fedplvm2024,fedktl2024,fedtgp2024} exchange compact class statistics such as class means or centroids to align global and local semantics, enabling lightweight personalization while maintaining frozen backbones.  

Despite their effectiveness, directly sharing prototypes poses significant privacy risks~\cite{wang2024fedpclcdr}.
Since prototypes condense features into domain fingerprints, they can closely resemble individual instances under small sample sizes, opening the door to membership inference and reconstruction attacks~\cite{park2024fedhide,darwish2025fedp3e,zhang2025mpft}. 
A standard defense is to enforce Local Differential Privacy (LDP)~\cite{wei2021user}.
In practice, clients first perform per-example $\ell_2$ clipping before prototype computation to bound sensitivity, then add calibrated isotropic Gaussian noise when transmitting prototypes, a recipe we call Isotropic Gaussian Prototype Perturbation (IGPP).
However, as illustrated in Figure~\ref{fig:issues}, IGPP suffers from two core drawbacks.
\blackcircnum{1} \textbf{Noise and discriminability mismatch.} Feature dimensions contribute unevenly to classification, with some carrying critical discriminability while others are largely redundant. Isotropic noise blindly over-perturbs the most informative dimensions and degrades class separability.
\blackcircnum{2} \textbf{$\ell_2$ clipping-threshold dilemma.} Prototype scales vary  based on class size and domain characteristics. A larger threshold reduces clipping distortion but demands massive noise injection, whereas a smaller threshold forces excessive shrinkage, irrevocably erasing semantic content.

Driven by these challenges, we introduce two modules that jointly account for prototype geometry and privacy.
\textbf{First}, to address \blackcircnum{1}, we derive dimension-wise discriminative scores from the intra- and inter-class variability of sample embeddings and use them to privately select a discriminative subspace. We then perform groupwise clipping on each embedding, construct prototypes from the clipped representations, and inject group-specific noise to reallocate the privacy budget toward task-relevant coordinates for better representation quality.
\textbf{Second}, to tackle \blackcircnum{2}, we append a differentiable soft-clipping layer to the local encoder and enforce prediction consistency between pre- and post-clipped features via knowledge distillation. This guidance breaks the norm–weight compensation shortcut and drives feature norms to concentrate near the clipping threshold, reducing information loss under the same privacy budget. 
 
In this paper, we propose a client-side privacy plug-in named \textbf{VPDR}, which comprises \textbf{V}ariance-adaptive \textbf{P}rototype Perturbation (VPP) and \textbf{D}istillation-guided Clipping \textbf{R}egularization (DCR). We leverage VPP to steer massive noise away from discriminative dimensions, preserving task-relevant structure while avoiding over-perturbation. We also employ DCR, designing a consistency regularizer that mitigates clipping-induced artifacts while keeping the baseline task loss unchanged. 
This synergy enables ProtoPFL frameworks equipped with VPDR to achieve more robust personalization than IGPP-based baselines.
Formal analysis and systematic evaluations under two privacy attacks further show that VPDR provides LDP guarantees no weaker than those of IGPP while achieving a superior privacy–utility trade-off.
Our main contributions are:
\begin{itemize}[itemsep=1pt, topsep=2pt, leftmargin=10pt]
  \item We fuse intra- and inter-class variance to quantify per-dimension discriminability, yielding, to our knowledge, the first dimension-wise adaptive protection mechanism for prototype release with LDP guarantees. 
  \item We design a distillation‑guided regularization that couples soft clipping with distillation consistency, suppressing clipping artifacts while enhancing both robustness and representation fidelity.
  \item Integrated into existing ProtoPFLs, VPDR yields consistent gains on three multi-domain benchmarks, with ablations and privacy attacks supporting the effectiveness of both VPP and DCR.
  Code is available at \url{https://github.com/yuCoryx/ProtoPFL_VPDR}.
\end{itemize}

\section{Related Work} 
\noindent\textbf{Differential Privacy in FL.}
A large body of work studies DP for gradients or model updates in standard FL~\cite{geyer2017differentially,Cheng2022fedblurs, bietti2022ppsgd, Shi2023fedsam,yang2023feddpa}. DP-FedAvg~\cite{geyer2017differentially} applies server-side Gaussian noise to achieve client-level central DP, while BLUS+LUS~\cite{Cheng2022fedblurs} and DP-FedSAM~\cite{Shi2023fedsam} reduce the accuracy degradation via gradient sparsification and flatness-aware training. Extensions to PFL include PPSGD~\cite{bietti2022ppsgd} and FedDPA~\cite{yang2023feddpa} adaptively mix or regularize personalized parameters under joint DP guarantees. 
To guard against untrusted servers, LDP mechanisms have also been introduced~\cite{wei2021user,he2023clustered,wang2025fedfr,cui2025aldp,tran2025privacy}. UDP-FL~\cite{wei2021user} perturbs each user’s update with calibrated Gaussian noise, and later schemes such as ACS-FL~\cite{he2023clustered}, FedFR-ADP~\cite{wang2025fedfr}, and ALDP-FL~\cite{cui2025aldp} refine this idea with adaptive clipping, noise scaling, and budget control, while all providing LDP guarantees for each client.

\noindent\textbf{Prototype-based PFL.}
ProtoPFL uses class prototypes to align feature spaces across clients. FedProto~\cite{FedProto2022} mitigates heterogeneity by exchanging prototypes, and FedPCL~\cite{fedpcl2022} improves semantic consistency through prototype-level contrastive learning. FPL~\cite{fplhuang2023}, FedGMKD~\cite{zhang2024fedgmkd}, and FedPLVM~\cite{fedplvm2024} integrate clustering and sparsity regularization to enhance robustness and generalization. More recent methods, including FedTGP~\cite{fedtgp2024}, FedKTL~\cite{fedktl2024}, and MPFT~\cite{zhang2025mpft}, leverage server-side training and knowledge transfer to reinforce global semantics. Together, these advances shift ProtoPFL from static prototype sharing toward dynamic semantic alignment, improving personalization across clients. 
Despite growing interest in privacy for prototype release, existing defenses~\cite{fedpcl2022,fedplvm2024,zhang2025mpft} often inject isotropic Gaussian noise into prototypes—typically without explicit per-example clipping or a full local DP accounting. 
As a result, \textit{they lack formal DP guarantees and overlook task-relevant prototype geometry, thereby degrading representation quality}.  We present the first LDP mechanism for ProtoPFL that combines groupwise clipping with adaptive noise allocation, offering privacy guarantees no weaker than isotropic perturbation and a better privacy–utility trade-off.

\begin{figure*}[t]
 \centering
 \includegraphics[width=1.0\linewidth]{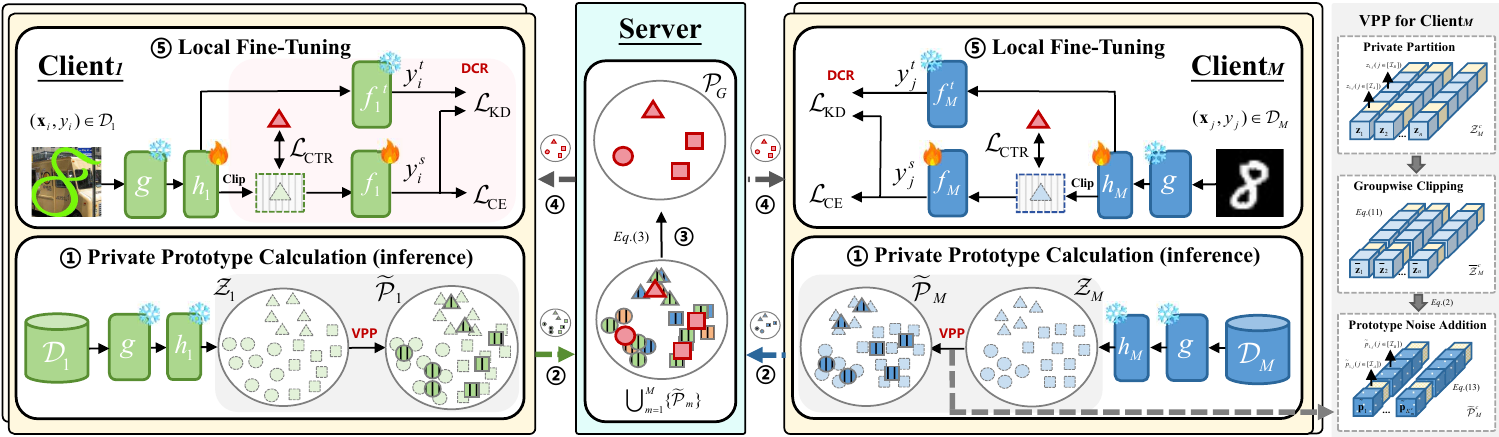} 
 \vspace{-4mm}
 \caption{\textbf{Architecture illustration} of the ProtoPFL with VPDR.  
\ding{172} Private Prototype Calculation: each client runs VPP (Section~\ref{sec:VPP}) to privately partition embeddings, apply groupwise clipping, and add adaptive noise to prototypes;
\ding{173} Upload: clients send privatized prototypes;  
\ding{174} Global Prototype Generation: the server aggregates or trains global prototypes;
\ding{175} Download: clients receive global prototypes;  
\ding{176} Local Fine-Tuning: each client fine-tunes locally with DCR (Section~\ref{sec:DCR}) for feature-norm stabilization alongside the task loss.}
 \label{fig:VPDR Overview}
\end{figure*}

\section{Preliminaries} 
\label{sec:preliminaries}
\subsection{Local Differential Privacy} 
\label{subsec:LDP}
We adopt client-side Local Differential Privacy (LDP) on each client's dataset~\cite{wang2019local,wei2021user,tran2025privacy} in this work, which protects individual training examples from an untrusted server. 

\begin{definition}[\((\epsilon, \delta)\)–LDP~\cite{wei2021user,tran2025privacy}]
A randomized mechanism $\mathcal{M}:\mathcal{X}\to\mathcal{R}$ satisfies $(\epsilon,\delta)$–LDP if for any adjacent $\mathcal{D}$ and $\mathcal{D}'$ differing by one sample and measurable $\mathcal{S}\subseteq\mathcal{R}$,
\(\Pr[\mathcal{M}(\mathcal{D})\in\mathcal{S}] \le e^{\epsilon}\Pr[\mathcal{M}(\mathcal{D}')\in\mathcal{S}] + \delta\),
where $\epsilon$ bounds the privacy loss and $\delta\in[0,1)$ allows a failure probability.
\end{definition} 

\begin{definition}[Sensitivity~\cite{wei2021user}]
Given a function \( f \), the \(\ell_p\) sensitivity of \( f \) is defined as
\(\Delta_f = \max_{\mathcal{D}\sim\mathcal{D}'} \| f(\mathcal{D}) - f(\mathcal{D}') \|_p\),
where $\|\cdot\|_p$ is the $\ell_1$ or $\ell_2$ norm.
\end{definition} 

\begin{definition}[Gaussian mechanism~\cite{wei2021user}]
For $\ell_2$ sensitivity $\Delta_f$, the Gaussian mechanism releases
\(\mathcal{M}(\mathcal{D})=f(\mathcal{D})+\mathcal{N}\!\bigl(0,(\sigma\Delta_f)^2\mathbf I_d\bigr),\) 
which satisfies $(\epsilon,\delta)$--LDP for a suitable noise multiplier $\sigma>0$.
\end{definition}
  
\begin{theorem}[Noise calibration~\cite{wei2021user}]
\label{thm:wei-noise}
In $T$-round FL, suppose each client $m$ releases
\(\widetilde{s}_m = s_m + \mathcal{N}\bigl(0,(\sigma\Delta_s)^2\mathbf{I}_d\bigr)\), 
where $s_m=f(\mathcal D_m)$ has $\ell_2$ sensitivity $\Delta_s$. Then there exists a constant $c_1>0$ such that \(\sigma \ge c_1\sqrt{T\ln(1/\delta)}/\epsilon\) ensures $(\epsilon,\delta)$–LDP over $T$ rounds for each client $m$.
\end{theorem}

\subsection{Oneshot Laplace Mechanism}\label{app:sec-OneshotLaplace}
The oneshot Laplace mechanism selects Top-$k$ coordinates by adding Laplace noise once and ranking the perturbed scores, ensuring DP without iterative peeling~\cite{Qiao2021OneshotTopK}. 

\begin{definition}[Oneshot Laplace for Top-$k$~\cite{Qiao2021OneshotTopK}]\label{def:oneshot-laplace}
Given queries \(f_1,\ldots,f_d\) on the dataset \( \mathcal{D} \) with common \(\ell_1\)-sensitivity \(\Delta_f\), draw \(b_i\!\sim\!\mathrm{Lap}(\lambda)\) and set \(y_i=f_i(\mathcal D)+b_i\). The mechanism then outputs the indices of the \( k \) largest \( y_i \)'s. 
\end{definition}

\begin{theorem}[Oneshot Laplace Privacy~\cite{Qiao2021OneshotTopK}]\label{thm:pure-dp}
If $\lambda \ge {2k\,\Delta_f}/{\epsilon}$, then the oneshot Laplace mechanism is $(\epsilon,0)$–DP.
\end{theorem}  

\section{Methodology}
\label{sec:method}
\subsection{Problem Statement and Motivation}   
\noindent \textbf{Personalized Federated Fine-Tuning.}
A federated system has \(M\) clients, each client \(m \in [M]\) with dataset \(\mathcal{D}_m\) of size $n_m$ drawn from a domain-shifted distribution \(P_m(\mathbf{x}\mid y)\). PFL seeks a client-specific model \(F_m\) that minimizes:
\begin{align}
    \min\nolimits_{F_m} \mathbb{E}_{(\mathbf x, y) \sim \mathcal{D}_m} \mathcal{L}\big(F_m(\mathbf x), y\big),
\end{align}
where \(\mathcal{L}(\cdot,\cdot)\) denotes the local loss. We write $F_m=f_m \circ h_m$, with a frozen backbone \(g\) and a trainable adapter \(A_m\) in feature extractor \(h_m\), followed by a classifier \(f_m\).

\vspace{0.5mm}
\noindent \textbf{Prototype-based Personalization.}
Each client maps private samples to \(d\)-dimensional embeddings via \(h_m\), yielding class-wise sets
\(\mathcal{Z}_m^c=\{\mathbf{z}_i^c=h_m(\mathbf{x}_i^c)\mid \mathbf{x}_i^c\in\mathcal{D}_m^c\}\).
For class \(c\), client \(m\) calculates a class-wise prototype set:
\begin{align}\label{eq:local-proto}
\{\mathbf{p}_m^{c,k}\}
\leftarrow \mathrm{Calculate}\bigl(\{\mathbf{z}_i^c \}\bigr),\;
\mathcal{P}_m^{c}=\big\{\mathbf{p}_m^{c,k}\big\}_{k=1}^{K_m^c},
\end{align}
where \(K_m^c\) is the number of local prototypes for class \(c\) on client \(m\). \(K_m^c=1\) recovers a single class mean~\cite{FedProto2022,fedpcl2022,fplhuang2023,fedtgp2024}, while \(K_m^c>1\) yields multiple local centroids~\cite{fedplvm2024,zhang2025mpft,zhang2024fedgmkd}.
The server aggregates $\{\mathcal{P}_m^c\}_{m=1}^M$ and generates:
\begin{align} \label{eq:glocal-proto}
\{\mathbf{p}_g^{c,k}\}
&\leftarrow \mathrm{Generate}\Bigl(\cup_{m=1}^M\mathcal{P}_m^c\Bigr),\;
\mathcal{P}_g^{c}=\big\{\mathbf{p}_g^{c,k}\big\}_{k=1}^{K_g^c},
\end{align}
where \(K_g^c\) is the number of global prototypes per class. \(K_g^c=1\) denotes a global prototype obtained either by averaging~\cite{FedProto2022,fedpcl2022} or by training from client prototypes~\cite{fedtgp2024}, and \(K_g^c>1\) corresponds to clustering into multiple global centroids~\cite{fedplvm2024,fplhuang2023}.
The global set \(\mathcal{P}_g=\bigl\{\mathcal{P}_g^{c}\bigr\}\) is then broadcast and used as global anchors in existing contrastive or classification objectives for multi-domain alignment.

\vspace{0.5mm}
\noindent \textbf{Isotropic Gaussian Prototype Perturbation (IGPP).} To achieve client-side LDP, each client (i) applies per-example $\ell_2$ clipping before any averaging or clustering,
\begin{align}
\overline{\mathbf z}_i^c
= \mathbf z_i^c\cdot \min\!\Bigl(1,\; R/\|\mathbf z_i^c\|_2\Bigr),
\end{align}
(ii) compute prototypes from $\{\overline{\mathbf z}_i^c\}$, and (iii) privatize each prototype release by adding isotropic Gaussian noise:
\begin{align}
\widetilde{\mathbf p}
=\mathbf p + \boldsymbol{\xi}, \;
\boldsymbol{\xi}\!\sim\!\mathcal N\!\bigl(\mathbf 0,(\sigma_{\mathrm{iso}}\Delta)^2\mathbf I_{d}\bigr)
\end{align}
where $\Delta$ equals $2R/n_m^c$ for class means and $2R/n_m^{c,k}$ for clustering of size $n_m^{c,k}$ within class $c$. The noise multiplier is calibrated to $(\epsilon,\delta)$–LDP over $T$ rounds via 
\begin{equation}
\sigma_{\mathrm{iso}} \geq c_1 {\sqrt{T \ln(1/\delta)}}/{\epsilon}
\quad\text{(Theorem~\ref{thm:wei-noise})}.    
\end{equation}

However, IGPP leaves two practical issues. 
\textbf{First}, applying isotropic noise perturbs all coordinates equally and can over-distort the most discriminative directions. 
\textbf{Second}, choosing a fixed $\ell_2$ clipping threshold is delicate.
A large threshold inflates noise via \(\Delta\) while a small one severely distorts feature semantics. 
To address both, we propose \textbf{VPDR}, a client-side plug-in for ProtoPFL. 
As sketched in Figure~\ref{fig:VPDR Overview}, our \textbf{V}ariance‐adaptive \textbf{P}rototype \textbf{P}erturbation (\textbf{VPP}) redistributes noise across dimensions to minimize information loss while preserving the same LDP guarantee, and our \textbf{D}istillation-guided \textbf{C}lipping \textbf{R}egularization (\textbf{DCR}) on each local adaptation stabilizes the per-sample norm bound.
A summary of notation is provided in Appendix~\ref{app:notaions}.

\subsection{Variance-Adaptive Prototype Perturbation}
\label{sec:VPP}
Isotropic noise overlooks the fact that feature dimensions matter differently for recognition: perturbing more informative dimensions harms utility more than redundant ones.
Intuitively, dimensions with low intra-class variance often encode stable class-specific semantics, whereas those with high inter-class variance are more effective at separating classes.
This motivates identifying an informative subspace and assigning it less perturbation.
The main challenge, however, is that selecting such a subspace is itself data-dependent and therefore privacy-sensitive.
We then split the budget over $T$ communication rounds via sequential composition: $(\epsilon, \delta) = (\epsilon_1, 0) + (\epsilon_2, \delta)$. Specifically, we allocate $\epsilon_1 = r\epsilon$ with $r \in (0, 1)$ to subspace selection and reserve $\epsilon_2 = (1-r)\epsilon$ for prototype release (allocation diagram in Appendix~\ref{app:Privacy Allocation}). 
This yields VPP, which adaptively reallocates perturbation while preserving the overall privacy guarantee.

\vspace{0.5mm}
\noindent \textbf{Discriminative Partition.}
For each $j\in[d]$, we define
\begin{align}\label{eq:v}
V^{\mathrm{tra}}_{j} {=} \sum_c (n_m^c-1) s_{c,j}^2,
\,  
V^{\mathrm{ter}}_{j} {=} \sum_c n_m^c(\mu_{c,j}-\mu_j)^2,
\end{align}
where the intra-class variance $s_{c,j}^2$, the class mean $\mu_{c,j}$ and the client-level mean $\mu_j$ are computed from encodings $\mathcal Z_m$.  
The ANOVA-normalized~\cite{scheffe1999analysis} discriminative score is: 
\begin{align}
\label{eq:anova-f}
S_j =\frac{V^{\mathrm{ter}}_{j}/(C-1)}{V^{\mathrm{tra}}_{j}/(n_m-C)+\zeta},
\; \zeta>0.
\end{align}
This normalization mitigates the effects of class-size imbalance and small-sample bias.  
\textit{Empirically}, we find that $S_j$ exhibits a strong linear correlation with the label mutual information $I(z_j;y)$ (Figure~\ref{fig:vpp_scatter}), with both Pearson and Spearman coefficients consistently exceeding $0.90$ across domains (Figure~\ref{fig:vpp_correlation}). This makes $S_j$ a reliable and lightweight proxy for dimension-wise discriminability.

\begin{figure}[t]
  \centering
  \begin{subfigure}[b]{0.495\linewidth}
    \includegraphics[width=\linewidth]{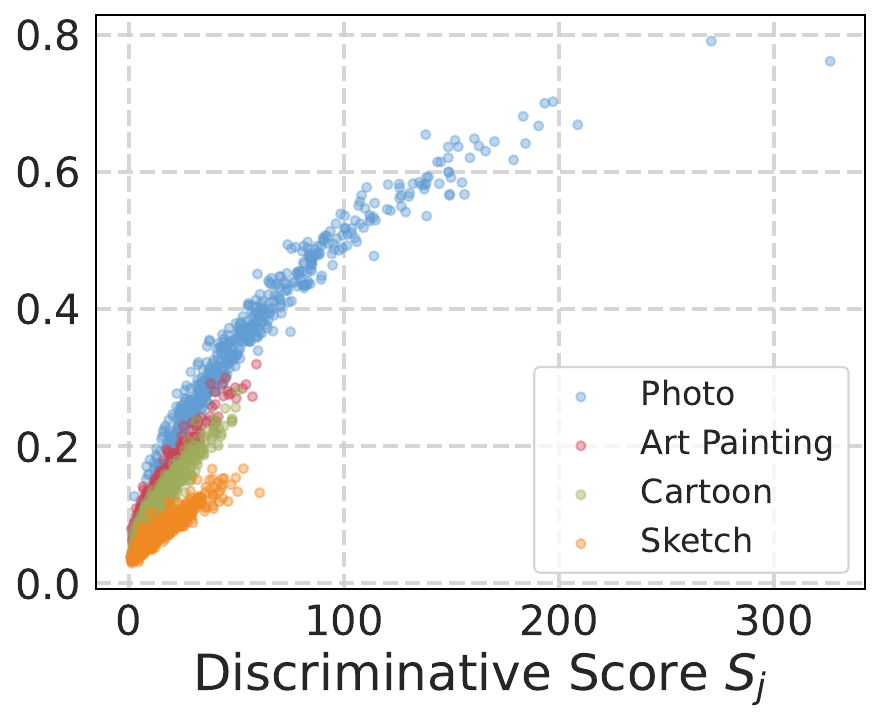}
    \caption{$S_j$ vs.\ $I(z_j;y)$ scatter.} 
    \label{fig:vpp_scatter}
  \end{subfigure}\hfill 
  \begin{subfigure}[b]{0.495\linewidth}
    \includegraphics[width=\linewidth]{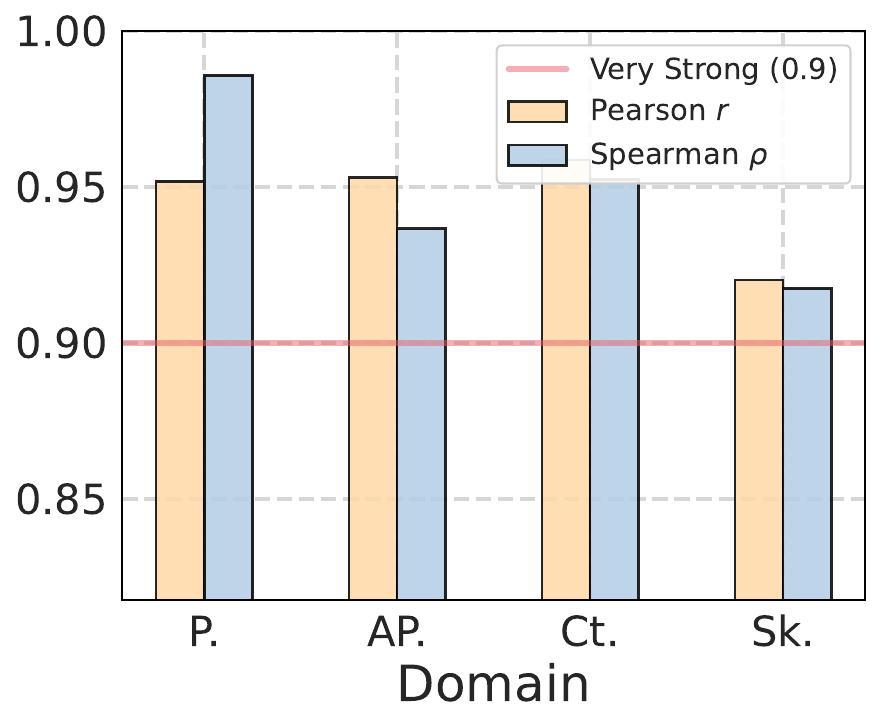}
    \caption{Per-domain correlations.}
    \label{fig:vpp_correlation}
  \end{subfigure}
  \vspace{-2mm}
  \caption{\textbf{Correlation} between the discriminative score $S_j$ and label mutual information $I(z_j;y)$ on PACS.} 
  \label{fig:vpp_motivation}
\end{figure} 

Although never uploaded, the selected index set determines the groupwise clipping bounds and noise covariance.
Thus, \textit{the released prototypes remain a function of this private selection} and must be included in privacy accounting.
To privately isolate informative dimensions, we clip raw scores to bound the $\ell_1$ sensitivity, $\overline{S}_j = \mathrm{clip}(S_j,0,H)
:= \min\{\max(S_j,0),\,H\}$, and apply the oneshot Laplace Top-$k$ mechanism~\cite{Qiao2021OneshotTopK}. Setting $k = d_A = \lceil\rho d\rceil$ for $\rho \in (0, 0.5]$ isolates the discriminative subspace $\mathcal{I}_A$, leaving the rest $d_B=d-d_A$ as the non-discriminative subspace $\mathcal{I}_B$. 
Calibrating the per-round Laplace noise scale as 
\begin{align}
\lambda \ge {2 d_A HT}/{\epsilon_1} 
\end{align}
guarantees $(\epsilon_1,0)$-LDP for the partition step over $T$ rounds as established by Theorem~\ref{thm:vpp-partition-pure}.

\vspace{0.5mm}
\noindent \textbf{Privatized Prototype Release.}
Given this partition, we decompose the feature vector as $\mathbf{z} = [\mathbf{z}_A; \mathbf{z}_B]$, where $\mathbf{z}_A = \mathbf{z}[\mathcal{I}_A]$ and $\mathbf{z}_B = \mathbf{z}[\mathcal{I}_B]$. For a fixed global clipping threshold $R$, we allocate groupwise bounds:
\begin{align}
R_A = R\kappa_A, R_B=R\kappa_B,
\end{align}
where $\kappa_A=\sqrt{{d_A}/{d}}, \kappa_B=\sqrt{{d_B}/{d}}$, so that $R_A^2+R_B^2=R^2$. We then independently apply groupwise $\ell_2$ clipping: 
\begin{align}\label{eq:group-clip}
&\overline{\mathbf z}_A = \mathbf z_A\cdot\min\bigl(1,R_A/\|\mathbf z_A\|_2\bigr), \nonumber \\
&\overline{\mathbf z}_B = \mathbf z_B\cdot\min\bigl(1,R_B/\|\mathbf z_B\|_2\bigr),
\end{align}
and form prototypes $\mathbf{p}=[\mathbf{p}_A;\mathbf{p}_B]$ by computing the mean or clustering on each group. This yields group sensitivities $\Delta_A=2R_A/n_m^c=\Delta\kappa_A$ and $\Delta_B=2R_B/n_m^c=\Delta\kappa_B$.  

Let $\sigma_{\mathrm{ref}}$ denote the noise multiplier of the reference isotropic Gaussian mechanism calibrated to $(\epsilon_2,\delta)$.
Theorem~\ref{thm:vpp-Prototype-Release} establishes that groupwise reallocation can be performed without weakening the privacy guarantee, provided that
$1/\sigma_A^2+1/\sigma_B^2\le 1/\sigma_{\mathrm{ref}}^2$.
We satisfy it by introducing weights $w_A, w_B > 0$ with $w_A+w_B=1$ and parameterizing the groupwise multipliers as 
\begin{equation}
\sigma_A={\sigma_{\mathrm{ref}}}/{\sqrt{w_A}},\,\sigma_B={\sigma_{\mathrm{ref}}}/{\sqrt{w_B}}.    
\end{equation}
For a parameter-free design, we set these weights solely based on the group dimensions:
\begin{align}\label{eq:weights}
w_A={\kappa_B}/{(\kappa_A+\kappa_B)},
w_B=1-w_A,    
\end{align} 

The privatized prototype release is then formulated as:
\begin{align}\label{eq:add-noise}
\widetilde{\mathbf p}
=\bigl[{\mathbf p}_A+\boldsymbol{\xi}_{A}&,\ {\mathbf p}_B+\boldsymbol{\xi}_{B}\bigr],
\end{align}  
\vspace{-5mm}
\begin{align}
\boldsymbol{\xi}_{A}\!\sim\!\mathcal N\bigl(\mathbf 0,(\sigma_A\Delta_A)^2\mathbf I_{d_A}\bigr), 
\boldsymbol{\xi}_{B}\!\sim\!\mathcal N\bigl(\mathbf 0,(\sigma_B\Delta_{B})^2\mathbf I_{d_B}\bigr). \nonumber
\end{align}
Under this construction, the release preserves the same $(\epsilon_2,\delta)$-LDP guarantee as the reference mechanism.

\vspace{0.5mm}
\noindent \textit{\textbf{Utility–privacy trade-off.}}
Under the sequential composition $(\epsilon,\delta)=(\epsilon_1,0)+(\epsilon_2,\delta)$,  reserving budget for subspace selection implies $\epsilon_2<\epsilon$, which increases the reference multiplier $\sigma_{\mathrm{ref}}$. VPP offsets this by actively steering noise away from task-relevant coordinates. Specifically, ensuring the informative subspace $\mathcal{I}_A$ is \textit{no noisier} than the isotropic reference requires $\sigma_A\Delta_A \le \sigma_{\mathrm{ref}}\Delta$. Substituting our weights~\eqref{eq:weights} and sensitivities, we obtain:
$$\kappa_A\sqrt{({\kappa_A+\kappa_B})/{\kappa_B}} \le 1
\;\Leftrightarrow\;
\kappa_A \le \kappa_B
\;\Leftrightarrow\;
0<\rho \le 0.5.$$
Consequently, restricting the selected discriminative dimensions to at most half of the feature space ensures that $\mathcal{I}_A$ receives no more noise than the isotropic $(\epsilon_2,\delta)$-LDP mechanism, while the excess noise is shifted to the redundant subspace $\mathcal{I}_B$.
As shown in Section~\ref{sec:Experiments}, this targeted preservation of discriminative features delivers a net utility gain that outweighs the overhead induced by privacy budget splitting.

\subsection{Distillation-Guided Clipping Regularization} 
\label{sec:DCR}
To ensure bounded sensitivity prior to adding calibrated Gaussian noise, clients typically apply per-example $\ell_2$ feature clipping before prototype construction. 
Hard clipping, however, is brittle: 
when a feature norm significantly exceeds the threshold ($\|\mathbf z_i\|_2 \gg R$), hard clipping aggressively shrinks the vector and irrevocably erases structural information. Conversely, for undersized features ($\|\mathbf z_i\|_2 \ll R$), the clipping operation is entirely ineffective, leaving small-magnitude vectors vulnerable to the injected noise.
DCR addresses this dilemma by regularizing the feature space dynamically during local training.

\noindent\textbf{Soft Clipping.} To mitigate sensitivity to $R$ in the prototype perturbation phase, we append a differentiable soft-clipping layer at the end of the feature encoder during local training:
\begin{align} \label{eq:soft clip}
\widehat {\mathbf z}_i
= \frac{R}{\|{\mathbf z}_i\|_2 + \gamma R}\cdot\,{\mathbf z}_i,
\quad 0 < \gamma \ll 1,    
\end{align} 
where \(\gamma\) controls the strength.  
This expands small-norm features toward $R$ and smoothly contracts large-norm features, avoiding the gradient discontinuities of hard clipping.

\noindent\textbf{Consistency Guidance.}
A potential failure mode is that the model counteracts soft clipping by inflating the classifier weights. As Figure~\ref{fig:norm} shows, pre-clipping feature norms tend to grow during mid-to-late training under soft clipping only.
To decouple this shortcut, we introduce an Exponential Moving Average (EMA) teacher–student distillation that enforces prediction consistency between pre- and post-clipped features without parameter coupling.  
Concretely, we split the local classifier into a trainable student head $f_m$ and a momentum teacher $f_m^{t}$ updated as
\begin{align}\label{eq:ema}
\theta_m^{t} = \beta {\theta}_m^{t} + (1-\beta) \theta_m,
\end{align}
where $\theta$ and ${\theta}_m^{t}$ denote the student and teacher parameters, respectively, and $\beta \in[0,1)$ is the EMA momentum. 
At each training step the teacher produces soft targets on the original feature, $y^t = f_m^t(\mathbf z)$, while the student predicts on the soft-clipped feature, $y^s = f_m(\overline{\mathbf z})$. We minimize the KL divergence between their temperature-scaled distributions: 
\begin{align} \label{eq:kd}
\mathcal{L}_{\text{KD}}=\mathrm{KL}\Bigl( \operatorname{softmax}\bigl(\frac{y^t}{\tau}\bigr) \big\|\operatorname{softmax}\bigl(\frac{y^s}{\tau}\bigr) \Bigr),
\end{align}
where $\tau$ is the distillation temperature. This consistency prevents the model from bypassing clipping via weight rescaling, effectively breaking the norm–weight tug-of-war.

Let $\mathcal L_{\text{BASE}}$ denote the ProtoPFL objective (e.g., cross-entropy for classification $\mathcal{L}_{\text{CE}}$ or an InfoNCE-style contrastive alignment $\mathcal{L}_{\text{CTR}}$ ~\cite{fedpcl2022,fedplvm2024,fplhuang2023,fedtgp2024}) as described in Figure~\ref{fig:VPDR Overview}.
The overall local fine-tuning objective is
\begin{align}\label{eq:overall-loss}
\mathcal{L}
= \mathcal{L}_{\text{BASE}}+\lambda_1 \mathcal{L}_{\mathrm{KD}}.
\end{align}
Here $\lambda_1$ is the KD weight. \textit{Empirically}, DCR consistently reduces norm drift and concentrates pre-clipping feature norms near $R$ during the mid-to-late training (Figure~\ref{fig:norm}). Also, it lowers the teacher–student logit discrepancy, providing evidence of stronger prediction consistency (Figure~\ref{fig:logit}). 
These behaviors make the framework less sensitive to the exact choice of $R$ while preserving task semantics.

\begin{figure}[t]
  \centering
  \begin{subfigure}[b]{0.5\linewidth}
    \includegraphics[width=\linewidth]{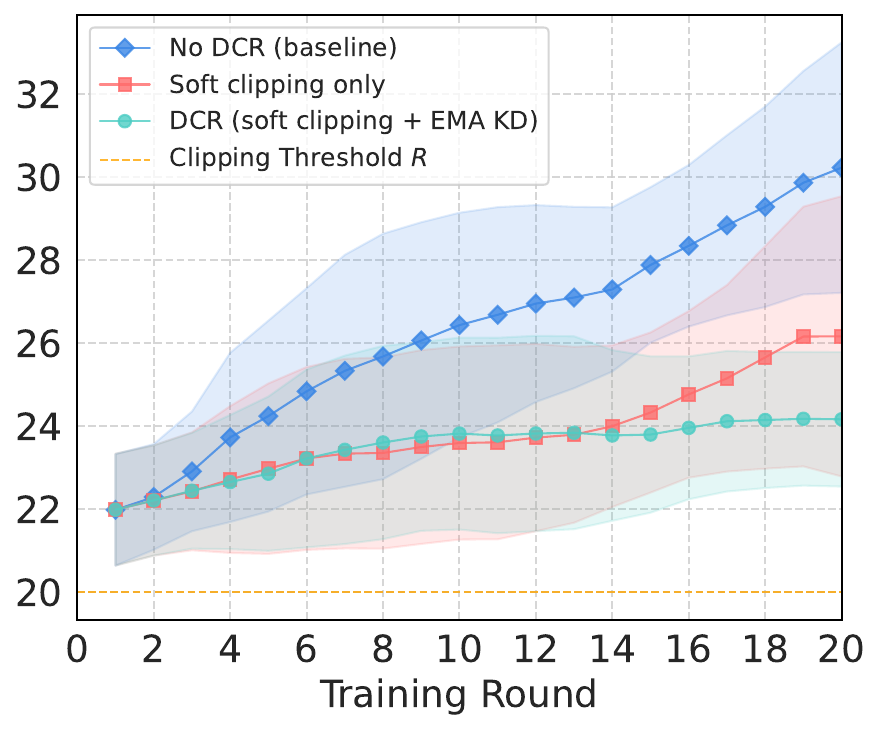}
    \caption{Pre-clipping feature norm.}
    \label{fig:norm}
  \end{subfigure}\hfill 
  \begin{subfigure}[b]{0.5\linewidth}
    \includegraphics[width=\linewidth]{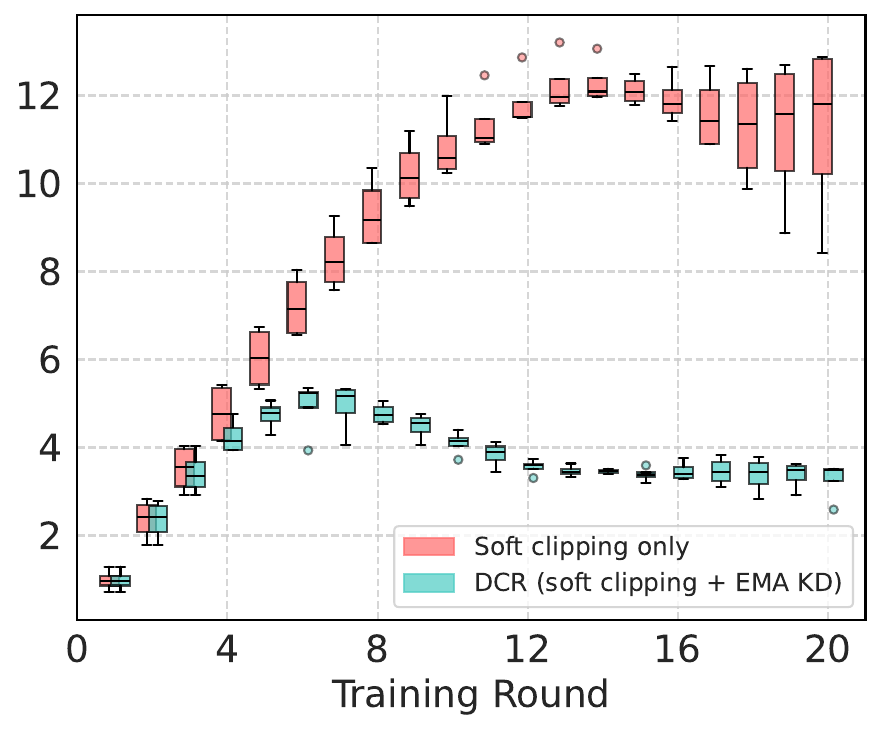}
    \caption{Teacher–student logit difference.}
    \label{fig:logit}
  \end{subfigure}
  
   \vspace{-2mm}
  \caption{\textbf{Evaluation of feature norm and logit difference} of FedPLVM~\cite{fedplvm2024} with IGPP on Office–Caltech.} 
  \label{fig:off_dcr}
\end{figure} 

\begin{algorithm}[t]
\caption{ProtoPFL with VPDR}
\label{alg:VPDR}
\textbf{Input}: Communication rounds $T$ , local epochs $E$, classes $C$ , clients $M$, private dataset $\mathcal D_m = \{{\mathbf x_i, y_i}\}_i^{n_m}$, global clipping threshold $R$, global noise multiplier $\sigma$\\
\textbf{Output}: Well-trained client models

\begin{algorithmic}[1]
\FOR{\(t = 1,2,\cdots, T\)} 
    \STATE \textcolor{MorandiBlue}{// Local Private Prototype Calculation} 
    \FOR{Client $m=1,2,\dots,M$}
        \STATE Encode features $z_i^c=h_m(\mathbf x_i^c)$
        \STATE Compute dimension scores $S$ via~\eqref{eq:v} and~\eqref{eq:anova-f};
        \STATE Privately partition to obtain $\mathcal I_A$ and $\mathcal I_B$;
        \STATE Perform groupwise clipping $\{\overline{z}_i^c\}$ by~\eqref{eq:group-clip};
       \FOR{Class $c=1,2,\dots,C$} 
            \STATE Get local prototypes and add noise by~\eqref{eq:add-noise}.
        \ENDFOR
        \STATE Upload the privatized prototype set $\widetilde{\mathcal{P}}_m$.
    \ENDFOR 
    \STATE \textcolor{MorandiBlue}{// Global Prototype Generation}
    \STATE Server aggregates or trains the global set $\mathcal{P}_g$ by~\eqref{eq:glocal-proto}.
    \STATE \textcolor{MorandiBlue}{// Local Personalized Fine-Tuning}
    \FOR{Client $m=1,2,\dots,M$}
        \FOR{\(e = 1,2,\cdots, E\)} 
            \STATE Apply soft clipping to features via~\eqref{eq:soft clip};
            \STATE Update the EMA teacher via~\eqref{eq:ema};
            \STATE Update adapter and classifier by Eq.~\eqref{eq:overall-loss}.
        \ENDFOR  
    \ENDFOR
\ENDFOR
\end{algorithmic}
\end{algorithm} 

\subsection{ProtoPFL with VPDR}
Algorithm~\ref{alg:VPDR} outlines ProtoPFL with VPDR.
In each round, clients first privately select a discriminative subspace and release privatized prototypes for uplink transmission. 
The server aggregates or trains these representations into global prototypes and broadcasts them back. With the global prototypes fixed, clients conduct local personalized fine-tuning via DCR while optimizing the base ProtoPFL objective. 
 
\vspace{0.5mm}
\noindent \textbf{Privacy Guarantees.}
VPDR confines all privacy-sensitive operations to the VPP during uplink transmission. Conversely, DCR is a purely local regularization that consumes no privacy budget. 
Under the budget split in Section~\ref{sec:VPP}, VPP is analyzed in two steps: Theorem~\ref{thm:vpp-partition-pure} provides $(\epsilon_1,0)$--LDP for private subspace selection, and Theorem~\ref{thm:vpp-Prototype-Release} provides $(\epsilon_2,\delta)$--LDP for prototype release.
By sequential composition, ProtoPFL with VPDR satisfies overall $(\epsilon,\delta)$--LDP. Full proofs are deferred to Appendix~\ref{app:More Proof of Privacy Analysis}.

\begin{restatable}[Partitioning privacy]{theorem}{PrivacyForPrivatePartition}
\label{thm:vpp-partition-pure}
Let $\{{S}_j\}_{j=1}^d$ be the scores from~\eqref{eq:anova-f} computed on client $m$, and define $\overline{S}_j=\mathrm{clip}(S_j,0,H)$ under adjacency on $\mathcal D_m$. In each round, execute the one-shot Laplace Top-$k$ mechanism ($k=d_A$) by drawing $b_j\sim\mathrm{Lap}(\lambda)$ to form $y_j = \overline{S}_j + b_j$, outputting the indices of the top-$k$ values. If $\lambda \ge 2d_AHT/\epsilon_1$, the $T$-round partitioning process satisfies $(\epsilon_1,0)$-LDP for client $m$.
\end{restatable}

Next, conditioned on the selected subspace, we calibrate noise for releasing the groupwise clipped prototypes.

\begin{restatable}[Release privacy]{theorem}
{PrivacyforPrototypeRelease}  
\label{thm:vpp-Prototype-Release}
Let the isotropic Gaussian release with sensitivity $\Delta$ achieve $(\epsilon_2,\delta)$–LDP after $T$ rounds using multiplier
\(\sigma_{\mathrm{ref}} \ge c_2 {\sqrt{T\ln(1/\delta)}}/{\epsilon_2}\).
Under VPP, use groupwise covariance $\Sigma_{\mathrm{VPP}}=\mathrm{diag}((\sigma_A\Delta_A)^2\mathbf I_{d_A},(\sigma_B\Delta_B)^2\mathbf I_{d_B})$
with $\Delta^2=\Delta_A^2+\Delta_B^2$. If the group multipliers satisfy
\begin{align}\label{eq:calibration}
\frac{1}{\sigma_A^2}+\frac{1}{\sigma_B^2}\le\frac{1}{\sigma_{\mathrm{ref}}^2},
\end{align}
then the $T$-round prototype release in VPP is also $(\epsilon_2,\delta)$–LDP for client $m$.
\end{restatable}

\begin{remark} \label{rmk:weights-choice}
A convenient choice satisfying~\eqref{eq:calibration} is 
$\sigma_A = \sigma_{\mathrm{ref}}/\sqrt{w}$ and 
$\sigma_B = \sigma_{\mathrm{ref}}/\sqrt{1-w}$ with $w\in(0,1)$. 
In VPP, we set $w$ according to~\eqref{eq:weights}, so the noise allocation is fully determined by the discriminability-aware partition and automatically adapts to the relative dimensionalities $d_A$ and $d_B$, without introducing any extra tunable hyperparameter.  
\end{remark}

\vspace{0.5mm}
\noindent \textbf{Complexity and Overheads.}
VPDR introduces only modest overhead to existing ProtoPFLs.
(1) \textbf{Computation:} Per communication round, VPP incurs an extra $\mathcal{O}(n_m d + d\log d)$ time over baseline construction, while noise addition remains $\mathcal{O}(d)$. 
DCR adds $\mathcal{O}(n_m d C E/B)$ training time for the teacher forward pass and KL divergence computation. 
(2) \textbf{Memory:} Overhead is minimal, requiring only $\mathcal{O}(d)$ to store the score vector and mask, alongside $\mathcal{O}(dC)$ for the momentum teacher head. 
(3) \textbf{Communication:} Complexity remains strictly unchanged at $\mathcal{O}(Cd)$, as clients solely transmit the final privatized prototypes.

\section{Experiments}\label{sec:Experiments}
\begin{table*}[t]
\caption{\textbf{Comparison} of Average Accuracy (AVG) and Standard Deviation (STD) under domain skew. The \textbf{best} is marked.} 
\vspace{-2mm}
\centering
\small
\renewcommand{\arraystretch}{1.25}
\fontsize{8}{7.66}\selectfont 
\setlength{\tabcolsep}{1pt}
\begin{tabularx}{\linewidth}
{l l||*{4}{>{\centering\arraybackslash}X}|*{2}{>{\centering\arraybackslash}X}|
 *{4}{>{\centering\arraybackslash}X}|*{2}{>{\centering\arraybackslash}X}|
 *{4}{>{\centering\arraybackslash}X}|*{2}{>{\centering\arraybackslash}X}}
\noalign{\hrule height 0.8pt}
\rowcolor{gray!25}
\multicolumn{2}{c||}{\textbf{Methods}} & \multicolumn{6}{c|}{\textbf{Digits}} & \multicolumn{6}{c|}{\textbf{Office–Caltech}} & \multicolumn{6}{c}{\textbf{PACS}} \\
\cline{3-8} \cline{9-14} \cline{15-20} 
\rowcolor{gray!25}
Framework & +LDP & MNI. & USPS & SVHN & SYN & AVG$\uparrow$ & STD$\downarrow$
& Amz.  & Cal.  & DSR.  & Web. & AVG$\uparrow$ & STD$\downarrow$
& P.    & AP.   & Ct.  & Sk.  & AVG$\uparrow$ & STD$\downarrow$ \\ 
\hline\hline
  & +IGPP  & 98.20 & 93.42 & 89.80 & 96.00 & 94.36 & 3.45 
           & 95.32 & 91.07 & 81.38 & 95.51 & 90.82 & 6.62
           & 97.70 & 92.71 & 91.52 & 80.37 & 90.58 & 7.31 \\
\multirow{-2}{*}{FedProto~\cite{FedProto2022}} 
  & \textbf{+VPDR}  & 98.31 & 94.87 & 92.33 & 98.70 & \cellcolor{vpdr}{\textbf{96.05}} & \cellcolor{vpdr}{\textbf{2.78}} 
           & 96.71 & 93.55 & 84.85 & 98.32 & \cellcolor{vpdr}{\textbf{93.36}} & \cellcolor{vpdr}{\textbf{6.01}} 
           & 98.80 & 95.60 & 91.88 & 84.59 & \cellcolor{vpdr}{\textbf{92.71}} & \cellcolor{vpdr}{\textbf{6.11}} \\
\hline
  & +IGPP   & 98.30 & 90.08 & 91.54 & 95.35 & 93.82 & 3.41
           & 95.29 & 92.41 & 64.52 & 94.92 & 86.79 & 13.53
           & 97.31 & 85.33 & 80.34 & 55.92 & 79.73 & 16.08 \\
\multirow{-2}{*}{FedPCL~\cite{fedpcl2022}}
  & \textbf{+VPDR}  & 98.92 & 92.63 & 93.15 & 96.95 & \cellcolor{vpdr}{\textbf{95.41}} & \cellcolor{vpdr}{\textbf{2.63}}
           & 96.33 & 92.86 & 74.97 & 96.61 & \cellcolor{vpdr}{\textbf{90.19}} & \cellcolor{vpdr}{\textbf{8.90}} 
           & 98.50 & 91.93 & 87.18 & 68.03 & \cellcolor{vpdr}{\textbf{86.91}} & \cellcolor{vpdr}{\textbf{11.84}} \\
\hline
  & +IGPP  & 98.15 & 93.67 & 91.89 & 96.45 & 95.04 & 2.44
           & 94.76 & 92.36 & 82.15 & 96.31 & 91.40 & 6.37
           & 98.10 & 94.13 & 91.16 & 82.08 & 91.37 & 6.28 \\
\multirow{-2}{*}{FPL~\cite{fplhuang2023}}   
  & \textbf{+VPDR}  & 98.10 & 95.60 & 93.40 & 97.60 & \cellcolor{vpdr}{\textbf{96.20}} & \cellcolor{vpdr}{\textbf{2.35}}
           & 96.34 & 93.72 & 85.02 & 99.65 & \cellcolor{vpdr}{\textbf{93.68}} & \cellcolor{vpdr}{\textbf{6.36}}
           & 98.86 & 96.78 & 93.15 & 83.93 & \cellcolor{vpdr}{\textbf{93.18}} & \cellcolor{vpdr}{\textbf{5.86}} \\
\hline
  & +IGPP  & 97.81 & 90.43 & 92.57 & 92.20 & 93.25 & 2.88
           & 95.61 & 91.26 & 82.05 & 96.61 & 91.38 & 6.64
           & 97.90 & 92.15 & 84.11 & 74.20 & 87.09 & 9.34 \\
\multirow{-2}{*}{FedPLVM~\cite{fedplvm2024}}  
  & \textbf{+VPDR}  & 97.71 & 92.93 & 94.57 & 93.70 & \cellcolor{vpdr}{\textbf{94.73}} & \cellcolor{vpdr}{\textbf{1.79}}
           & 97.30 & 92.23 & 84.76 & 99.21 & \cellcolor{vpdr}{\textbf{93.38}} & \cellcolor{vpdr}{\textbf{6.07}}
           & 98.80 & 93.71 & 86.04 & 77.63 & \cellcolor{vpdr}{\textbf{89.05}} & \cellcolor{vpdr}{\textbf{8.38}} \\
\hline
  & +IGPP  & 97.62 & 93.87 & 91.69 & 97.20 & 95.10 & 2.44 
           & 96.34 & 92.38 & 83.97 & 96.55 & 92.31 & \textbf{5.88}
           & 98.80 & 96.62 & 92.09 & 84.29 & 92.95 & 5.80 \\
\multirow{-2}{*}{FedTGP~\cite{fedtgp2024}}  
  & \textbf{+VPDR}  & 97.84 & 95.32 & 93.24 & 98.75 & \cellcolor{vpdr}{\textbf{96.29}} & \cellcolor{vpdr}{\textbf{2.14}}
           & 97.34 & 96.61 & 86.21 & 100 & \cellcolor{vpdr}{\textbf{95.04}} & \cellcolor{vpdr}{6.19}
           & 100 & 97.32 & 94.83 & 87.01 & \cellcolor{vpdr}{\textbf{94.79}} & \cellcolor{vpdr}{\textbf{5.64}} \\
\hline
  & +IGPP   & 97.61 & 94.27 & 92.95 & 98.00 & 95.71 & \textbf{2.16}
           & 95.76 & 92.52 & 84.03 & 96.61 & 92.23 & \textbf{5.74}
           & 99.63 & 96.33 & 93.16 & 84.50 & 93.41 & 5.89 \\
\multirow{-2}{*}{MPFT~\cite{zhang2025mpft}}   
  & \textbf{+VPDR}  & 98.64 & 94.77 & 93.78 & 99.05 & \cellcolor{vpdr}{\textbf{96.56}} & \cellcolor{vpdr}{2.32}
           & 97.81 & 94.15 & 86.58 & 98.31 & \cellcolor{vpdr}{\textbf{94.71}} & \cellcolor{vpdr}{5.83}
           & 99.40 & 97.70 & 95.65 & 86.88 & \cellcolor{vpdr}{\textbf{94.91}} & \cellcolor{vpdr}{\textbf{5.04}} \\
\noalign{\hrule height 0.8pt}
\end{tabularx}
\label{tab:all_datasets_eps1_split}
\end{table*}

\begin{figure*}[htbp]
    \centering
    \begin{subfigure}[b]{0.198\textwidth}
        \centering
        \includegraphics[width=\textwidth]{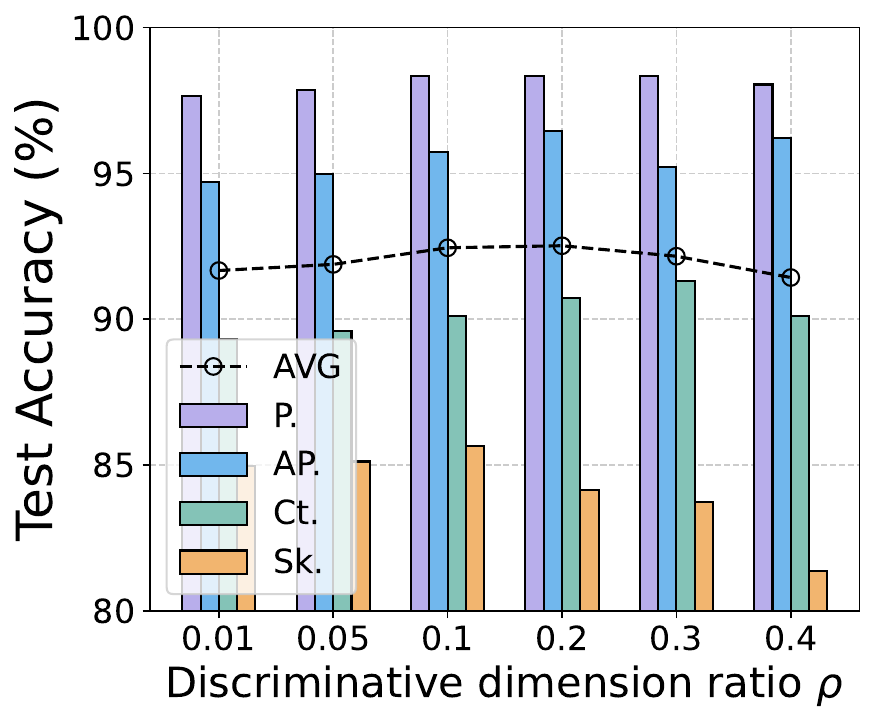}
        \caption{The effect of \(\rho\).}
    \end{subfigure}\hfill
    \begin{subfigure}[b]{0.198\textwidth}
        \centering
        \includegraphics[width=\textwidth]{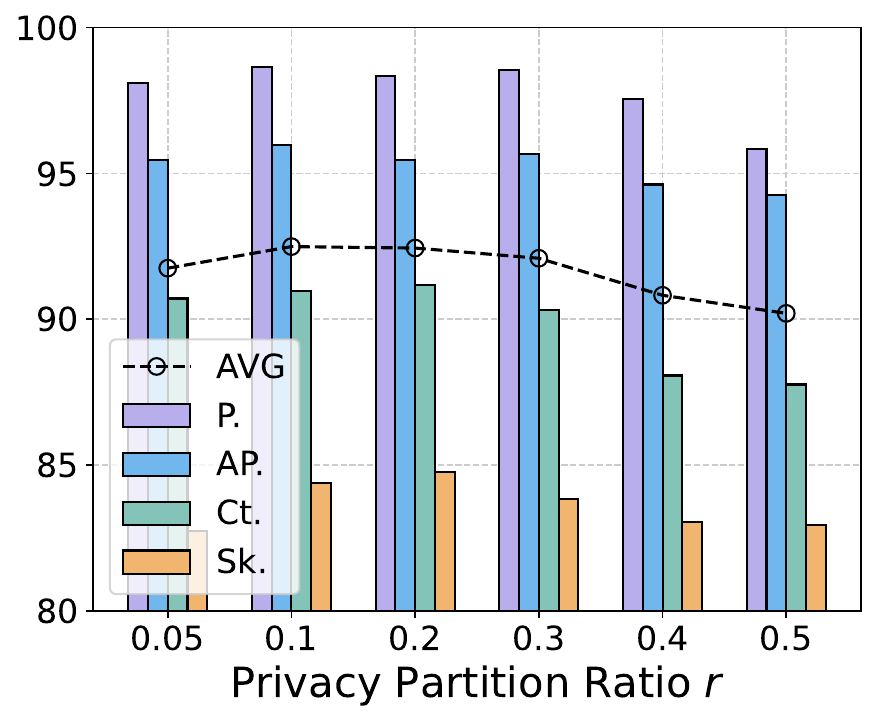}
        \caption{The effect of \(r\).}
    \end{subfigure}\hfill
    \begin{subfigure}[b]{0.198\textwidth}
        \centering
        \includegraphics[width=\textwidth]{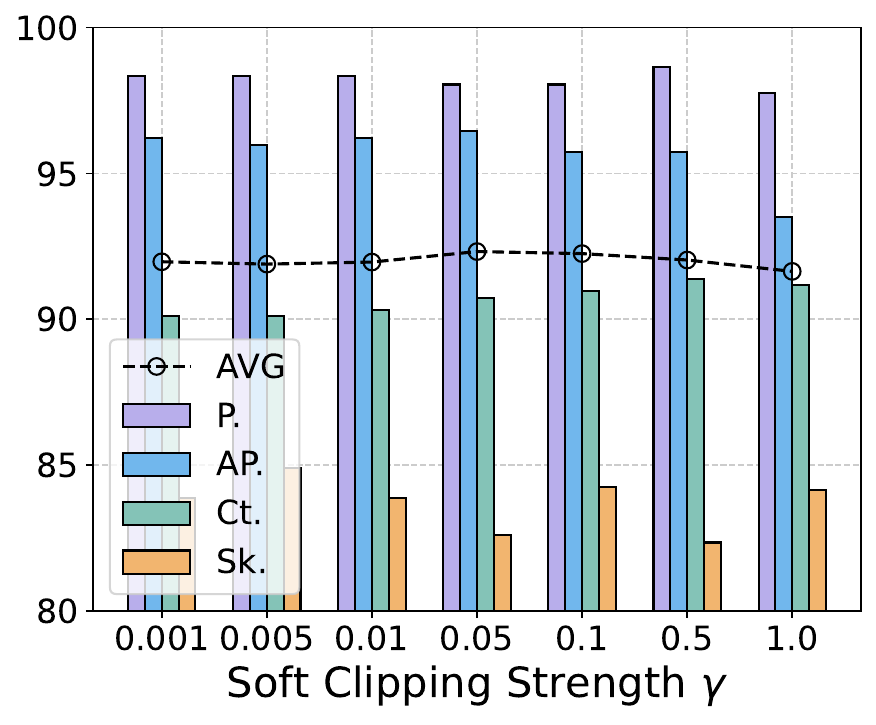}
        \caption{The effect of \(\gamma\).}
    \end{subfigure}
    \begin{subfigure}[b]{0.198\textwidth}
        \centering
        \includegraphics[width=\textwidth]{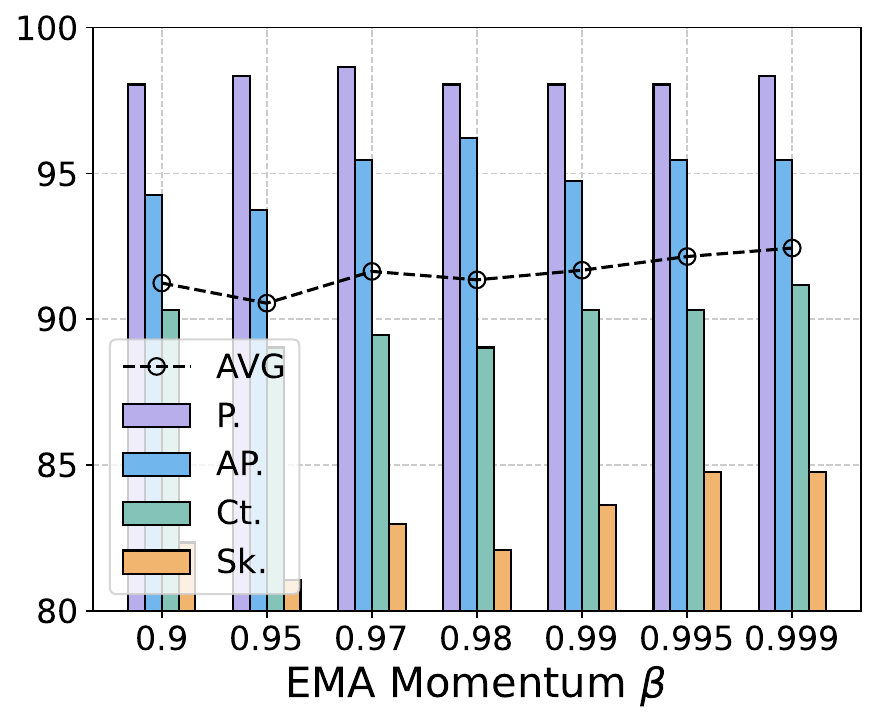}
        \caption{The effect of \(\beta\).}
    \end{subfigure}
    \begin{subfigure}[b]{0.198\textwidth}
        \centering
        \includegraphics[width=\textwidth]{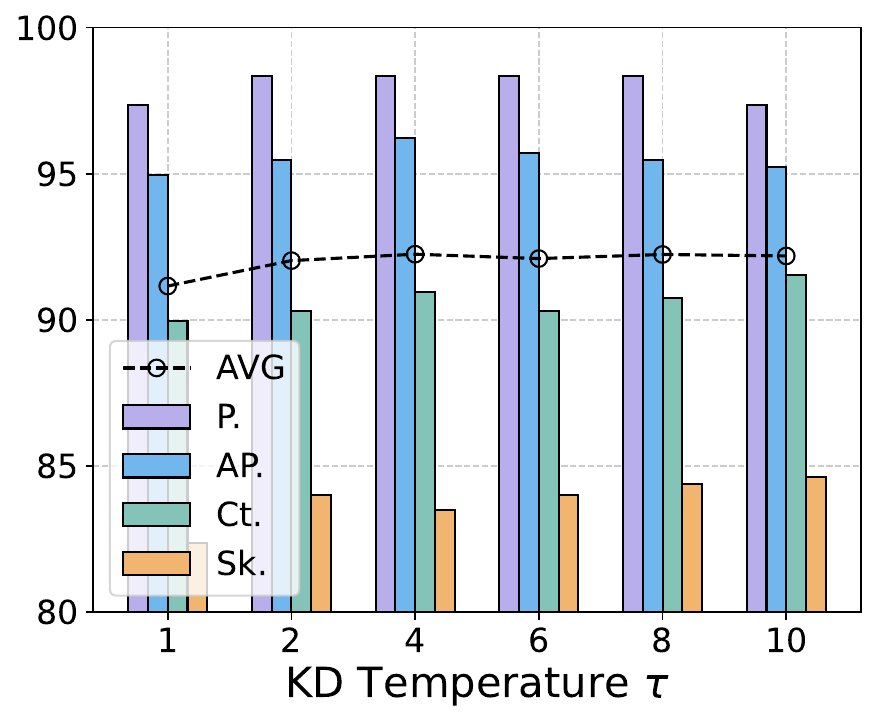}
        \caption{The effect of \(\tau\).}
    \end{subfigure}
    \vspace{-6mm}
    \caption{\textbf{Hyperparameter sensitivity} of FedProto~\cite{FedProto2022} framework with VPDR on PACS.}
    \label{fig:hyper_sensitivity}
\end{figure*}

\begin{table}[t]
\centering 
\caption{\textbf{Ablation results} of FedProto~\cite{FedProto2022} framework.}
\label{tab:ablation_modules}
\vspace{-2mm}
\setlength{\tabcolsep}{3pt}
\renewcommand{\arraystretch}{1.25}
\fontsize{8}{7.66}\selectfont
\begin{tabularx}{\columnwidth}{YY||YYYY|YY} 
\noalign{\hrule height 0.8pt}  
\rowcolor{gray!25} 
& & \multicolumn{6}{c}{\textbf{Office-Caltech}} \\ 
\cline{3-8}
\rowcolor{gray!25}
\multirow{-2}{*}{VPP} & \multirow{-2}{*}{DCR} & Amz. & Cal. & DSR & Web.  & AVG$\uparrow$   & STD $\downarrow$\\
\hline\hline
                &                & 95.32 & 91.07 & 81.38 & 95.51 & 90.82 & 6.62 \\
\checkmark      &                & 96.96 & 92.88 & 84.53 & 97.85 & 93.06 & 6.08 \\
                & \checkmark     & 97.08 & 92.60 & 84.57 & 97.31 & 92.89 & \textbf{5.96} \\
\checkmark      & \checkmark     & 96.71 & 93.55 & 84.85 & 98.32 &\cellcolor{vpdr} \textbf{93.36} & \cellcolor{vpdr}6.01 \\
\noalign{\hrule height 0.8pt} 
\rowcolor{gray!25}
& &  \multicolumn{6}{c}{\textbf{PACS}} \\
\cline{3-8}
\rowcolor{gray!25}
\multirow{-2}{*}{VPP} & \multirow{-2}{*}{DCR}  & P.    & AP.   & Ct.   & Sk.   & AVG$\uparrow$   & STD $\downarrow$\\
\hline\hline
                &                & 97.70 & 92.71 & 91.52 & 80.37 & 90.58 & 7.31 \\
\checkmark      &                & 98.91 & 93.15 & 92.18 & 83.31 & 91.89 & 6.44 \\
                & \checkmark     & 98.62 & 93.70 & 91.82 & 83.51 & 91.91 & 6.29 \\
\checkmark      & \checkmark     & 98.80 & 95.60 & 91.88 & 84.59 & \cellcolor{vpdr}\textbf{92.71} & \cellcolor{vpdr}\textbf{6.11}\\
\noalign{\hrule height 0.8pt}
\end{tabularx}
\end{table}

\subsection{Experimental Setup}
\vspace{0.5mm}
\noindent \textbf{Datasets.} 
We conduct experiments on three domain‐skew benchmarks: Digits~\cite{hull1994database,lecun1998gradient,netzer2011reading,roy2018effects}, which comprises MNIST (MNI.), USPS, SVHN and synthetic digits (SYN); Office–Caltech~\cite{gong2012geodesic}, consisting of Caltech (Cal.), Amazon (Amz.), DSLR (DSR.) and Webcam (Web.); PACS~\cite{li2017deeper} covering Photo (P.), Art Painting (AP.), Cartoon (Ct.) and Sketch (Sk.). 
Details are provided in Appendix~\ref{app:dataset}.

\vspace{0.5mm}
\noindent \textbf{Model.}  
All experiments use a ViT-small~\cite{dosovitskiy2021vit} backbone with a hidden size $d$ of 512 and frozen pretrained weights. 
We insert lightweight adapters after each Transformer block, and train only the adapters and the classifier. 
 
\vspace{0.5mm}
\noindent \textbf{Counterparts.}
We compare six ProtoPFL frameworks, all equipped with IGPP: FedProto~\cite{FedProto2022}, FedPCL~\cite{fedpcl2022}, FPL~\cite{fplhuang2023}, FedPLVM~\cite{fedplvm2024}, FedTGP~\cite{fedtgp2024}, and MPFT~\cite{zhang2025mpft}.
Personalization is achieved by exchanging only privatized prototypes. For MPFT, the server learns an adapter from the privatized prototypes and transmits it to clients, which is the post-processing of DP~\cite{dwork2006calibrating} and no extra privacy cost.

\vspace{0.5mm}
\noindent \textbf{Implementation Details.}
Each benchmark has four domain-specific clients~\cite{fedpcl2022,fedplvm2024},  training on $10\%$ of Digits and $30\%$ of Office–Caltech and PACS. We run communication rounds $T=20$ with local epochs $E=2$ and batch size $256$. Local optimization is AdamW with learning rate $10^{-3}$, weight decay $10^{-5}$, and momentum $0.9$. ProtoPFL hyperparameters are listed in Appendix~\ref{app:Beseline Hyperparameters}. Unless otherwise noted, we set $(\epsilon,\delta){=}(1,10^{-5})$ and select the clipping radius $R$ by grid search over $\{5,10,15,20\}$.
VPDR defaults are $H=0.1$, $r = 0.1$, $\rho = 0.2$, $\gamma = 0.05$, $\beta = 0.999$, $\tau= 4$ and $\lambda_1$= 0.05. 
We evaluate each personalized model on its own local data and report the average test accuracy. 

\subsection{Performance Comparison}\label{sec:Performance Comparison}
Table \ref{tab:all_datasets_eps1_split} shows that VPDR consistently improves average accuracy (AVG) compared with IGPP across all ProtoPFLs and benchmarks. Improvements are most pronounced on the more heterogeneous PACS and Office–Caltech. While a few per-domain scores fluctuate, the per-dataset averages for every framework increase, and dispersion (STD) typically shrinks on harder regimes, indicating that VPDR’s variance-adaptive noise and clipping stabilization primarily reduce multi-domain fluctuation where it matters most. 
Additional results and t-SNE plots are in Appendix~\ref{app:More Results of Performance Comparison}.

\subsection{Diagnostic Analysis}\label{sec:Diagnostic Analysis}
\noindent \textbf{Hyperparameter Study.}
Figure~\ref{fig:hyper_sensitivity} illustrates the impact of hyperparameters on PACS in FedProto~\cite{FedProto2022}. The optimal range for the privacy partition ratio $r$ lies between $0.1$ and $0.2$, while the discriminative subspace ratio $\rho$ performs best around $0.2$. The optimal soft clipping strength $\gamma$ is $0.05$, and the EMA momentum $\beta$ achieves its best performance at $0.999$. The KD temperature $\tau$ is most effective at $4$, with little variation between $4$ and $8$. These trends confirm that VPDR delivers robust performance across a wide spectrum and that our selected defaults lie close to the optimum.

\vspace{0.5mm}
\noindent \textbf{Ablation Study.} 
We ablate VPP and DCR within FedProto~\cite{FedProto2022} as shown in Table~\ref{tab:ablation_modules}.  
Starting from IGPP, adding either VPP or DCR alone improves performance on both Office–Caltech and PACS. Using both together yields the best average accuracy, confirming complementary gains.

\begin{table}[t]
\centering
\caption{\textbf{Average accuracy} (\%) on CIFAR-10 under label skew.}
\label{tab:label_skew_comparison}
\vspace{-2mm}
\setlength{\tabcolsep}{1pt}
\fontsize{8}{7.66}\selectfont
\renewcommand{\arraystretch}{1.25}
\begin{tabularx}{\columnwidth}{l l || *{5}{>{\centering\arraybackslash}X}}
\noalign{\hrule height 0.8pt}
\rowcolor{gray!25}
\multicolumn{2}{l||}{Methods} & $\alpha=0.1$ & $\alpha=0.5$ & $\alpha=1$ & $\alpha=5$ & $\alpha=10$ \\
\hline\hline
\multirow{2}{*}{FedProto~\cite{FedProto2022}} 
  & +IGPP  & 42.76 & 82.31 & \textbf{90.23} & 96.15 & 96.30 \\
  & \textbf{+VPDR}  & \textbf{44.14} & \textbf{85.00} & 89.99 & \textbf{96.34} & \textbf{96.55} \\
\hline
\multirow{2}{*}{FedPCL~\cite{fedpcl2022}}     
  & +IGPP  & 18.22 & 73.63 & 87.78 & 95.55 & \textbf{95.67} \\
  & \textbf{+VPDR}  & \textbf{20.92} & \textbf{73.98} & \textbf{88.29} & \textbf{95.63} & 95.61 \\
\hline
\multirow{2}{*}{FPL~\cite{fplhuang2023}}      
  & +IGPP  & 28.42 & 69.14 & 88.19 & 95.39 & 95.58 \\
  & \textbf{+VPDR}  & \textbf{37.18} & \textbf{80.14} & \textbf{88.73} & \textbf{96.04} & \textbf{96.29} \\
\hline
\multirow{2}{*}{FedPLVM~\cite{fedplvm2024}}   
  & +IGPP  & 37.33 & 76.95 & \textbf{87.21} & 95.06 & 95.32 \\
  & \textbf{+VPDR}  & \textbf{40.17} & \textbf{77.42} & 82.76 & \textbf{95.47} & \textbf{95.34} \\
\hline
\multirow{2}{*}{FedTGP~\cite{fedtgp2024}}     
  & +IGPP   & 40.71 & 79.92 & 87.68 & 95.28 & 95.47 \\
  & \textbf{+VPDR}  & \textbf{43.03} & \textbf{82.46} & \textbf{89.61} & \textbf{96.32} & \textbf{96.42} \\
\hline
\multirow{2}{*}{MPFT~\cite{zhang2025mpft}}    
  & +IGPP   & 46.79 & 79.06 & 91.23 & \textbf{96.64} & \textbf{96.71} \\
  & \textbf{+VPDR}  & \textbf{47.33} & \textbf{80.17} & \textbf{91.26} & 96.53 & 96.64 \\
\noalign{\hrule height 0.8pt}
\end{tabularx}
\end{table}
  
\begin{table*}[t]
\centering
\caption{\textbf{Privacy Attack Results} of FedProto~\cite{FedProto2022} on Office-Caltech under varying privacy budgets.}
\label{tab:office-attacks}
\vspace{-2mm}
\setlength{\tabcolsep}{3pt}
\fontsize{8}{7.66}\selectfont
\renewcommand{\arraystretch}{1.25}
\begin{tabularx}{\textwidth}{l|l||YYY|YYYY}
\noalign{\hrule height 0.8pt}
\rowcolor{gray!25}
& & \multicolumn{3}{c|}{\textbf{Feature Reconstruction (FSH)}} & \multicolumn{4}{c}{\textbf{Membership Inference (MIA)}} \\
\cline{3-5}\cline{6-9}
\rowcolor{gray!25}
\multirow{-2}{*}{$\epsilon$}  & \multirow{-2}{*}{LDP}
& Cos Sim\scriptsize{$\downarrow$} & cFFD\scriptsize{$\uparrow$} & Top-1 Hit(\%)$\downarrow$
& ROC-AUC\scriptsize{$\rightarrow0.5$} & TPR@1\%FPR\scriptsize{$\downarrow$} & Advantage\scriptsize{$\rightarrow0$} & F1 Score\scriptsize{$\downarrow$} \\
\hline\hline
- & NoLDP   & 0.9999$\pm$0.0000 & 257.7$\pm$65.0 & 100.00$\pm$0.00 & 0.6729$\pm$0.0124 & 0.0298$\pm$0.0200 & 0.4458$\pm$0.0170 & 0.7881$\pm$0.0082 \\
\hline
\multirow{3}{*}{1} 
& +IGPP   & 0.6437$\pm$0.0022 & 2214.6$\pm$21.0 & 20.29$\pm$3.52 & 0.5079$\pm$0.0340 & 0.0043$\pm$0.0045 & 0.2669$\pm$0.0300 & 0.7317$\pm$0.0085 \\
& \textbf{+VPP}   & 0.6456$\pm$0.0026 & 2197.6$\pm$18.2 & 21.22$\pm$4.04 & 0.4992$\pm$0.0293 & 0.0040$\pm$0.0050 & 0.2536$\pm$0.0279 & 0.7282$\pm$0.0119 \\
& \cellcolor{vpdr}\textbf{+VPDR}
& \cellcolor{vpdr}0.6456$\pm$0.0032
& \cellcolor{vpdr}2207.1$\pm$27.7
& \cellcolor{vpdr}20.62$\pm$3.33
& \cellcolor{vpdr}0.5055$\pm$0.0275
& \cellcolor{vpdr}0.0042$\pm$0.0044
& \cellcolor{vpdr}0.2649$\pm$0.0299
& \cellcolor{vpdr}0.7321$\pm$0.0109 \\
\hline
\multirow{3}{*}{2} & +IGPP   & 0.6710$\pm$0.0025 & 2038.5$\pm$22.7 & 33.23$\pm$3.59 & 0.5081$\pm$0.0336 &                       0.0066$\pm$0.0036 & 0.2855$\pm$0.0332 & 0.7399$\pm$0.0073 \\
& \textbf{+VPP}   & 0.6706$\pm$0.0025 & 2040.7$\pm$26.6 & 34.44$\pm$5.83 & 0.5097$\pm$0.0228 & 0.0065$\pm$0.0016 & 0.2834$\pm$0.0244 & 0.7406$\pm$0.0084 \\
& \cellcolor{vpdr}\textbf{+VPDR}
& \cellcolor{vpdr}0.6718$\pm$0.0029
& \cellcolor{vpdr}2029.4$\pm$26.0
& \cellcolor{vpdr}33.75$\pm$2.90
& \cellcolor{vpdr}0.5089$\pm$0.0299
& \cellcolor{vpdr}0.0063$\pm$0.0031
& \cellcolor{vpdr}0.2846$\pm$0.0309
& \cellcolor{vpdr}0.7404$\pm$0.0117 \\
\noalign{\hrule height 0.8pt}
\end{tabularx}
\end{table*}

\vspace{0.5mm}
\noindent \textbf{Scalability under Label Skew.}
We induce client-level label skew on CIFAR-10~\cite{krizhevsky2009learning} using a Dirichlet distribution with concentration $\alpha \in \{0.1, 0.5, 1, 5, 10\}$, where smaller $\alpha$ means higher skew. 
Table~\ref{tab:label_skew_comparison} ($M=20$, $\epsilon=1$) shows that VPDR yields the largest gains under extreme skew ($\alpha \le 0.5$). While this advantage narrows toward the IGPP as data becomes balanced, it highlights that VPP's discriminative focus and DCR's norm stabilization thrive under high heterogeneity. 
Furthermore, we verify VPDR's robustness under joint model--data heterogeneity and its generalization to text modality in Appendices~\ref{app:model-data-skew} and~\ref{app:generalization-text}.

\subsection{Defense Assessment Under Attacks}
\label{sec:attack}
We consider a realistic deployment threat model in which an adversary can access client-uploaded prototypes, either as a semi-honest server or as an external eavesdropper. 
We evaluate two attacks against all clients per communication round: Feature-Space Hijacking (FSH)~\cite{zhang2025mpft} and Membership Inference Attacks (MIA)~\cite{shokri2017membership}. 
FSH actively reconstructs raw client data by optimizing dummy inputs to match targeted prototypes. A successful defense yields low cosine similarity, high classifier feature fr\'echet distance (cFFD), and low Top-1 Hit rates. Conversely, MIA passively infers training set membership by thresholding sample-to-prototype distances, evaluated via ROC-AUC (where values near 0.5 indicate random guessing), TPR@1\%FPR, advantage, and F1 score.
Detailed attack settings are deferred to Appendix~\ref{app:Details of Privacy Attacks}.

Table~\ref{tab:office-attacks} reports the attack outcomes on Office-Caltech under $\epsilon \in \{1, 2\}$, with identical trends observed on Digits and PACS (Appendix~\ref{app:Details of Privacy Attacks}). 
Compared to NoLDP, all LDP mechanisms drastically degrade the attacker's capabilities.  Crucially, VPP and VPDR defense metrics match the IGPP baseline within a single standard deviation. At $\epsilon=1$, all three mechanisms drive MIA ROC-AUC to chance ($\sim 0.50$) and throttle the FSH Top-1 Hit rate to roughly $20\%$. 
These results confirm that VPDR's variance-adaptive reallocation and soft clipping do not leak structural vulnerabilities, achieving superior utility under the same practical defense strength as isotropic Gaussian perturbation. 
More results varying the $d$ and $r$ are in Appendix~\ref{app:effect-dr}.

\subsection{Computational Cost} 
Figure~\ref{fig:time} reports the per-round, per-client wall-clock time on Office-Caltech and PACS, decomposed into Prototype Generation (ProtoGen) and Local Fine-Tuning (FT). 
Across six ProtoPFL frameworks, VPDR increases total runtime by an average of only $0.28$s ($8.1\%$) compared to the IGPP baseline, with phase-level increases of $+0.09$s for ProtoGen and $+0.20$s for FT. 
Table~\ref{tab:time_overhead} details this cost for FedProto on Office-Caltech. VPDR operations incur marginal overheads: $2.2\%$ in ProtoGen (variance statistics and Laplace Top-$k$) and $0.2\%$ in FT (soft clipping, teacher-head forward, KL divergence term). Ultimately, VPDR delivers accuracy gains with acceptable or even negligible computational overhead and no additional server-side work.

\begin{figure}[t]
  \centering 

  \begin{subfigure}[b]{0.88\linewidth}
    \includegraphics[width=\linewidth]{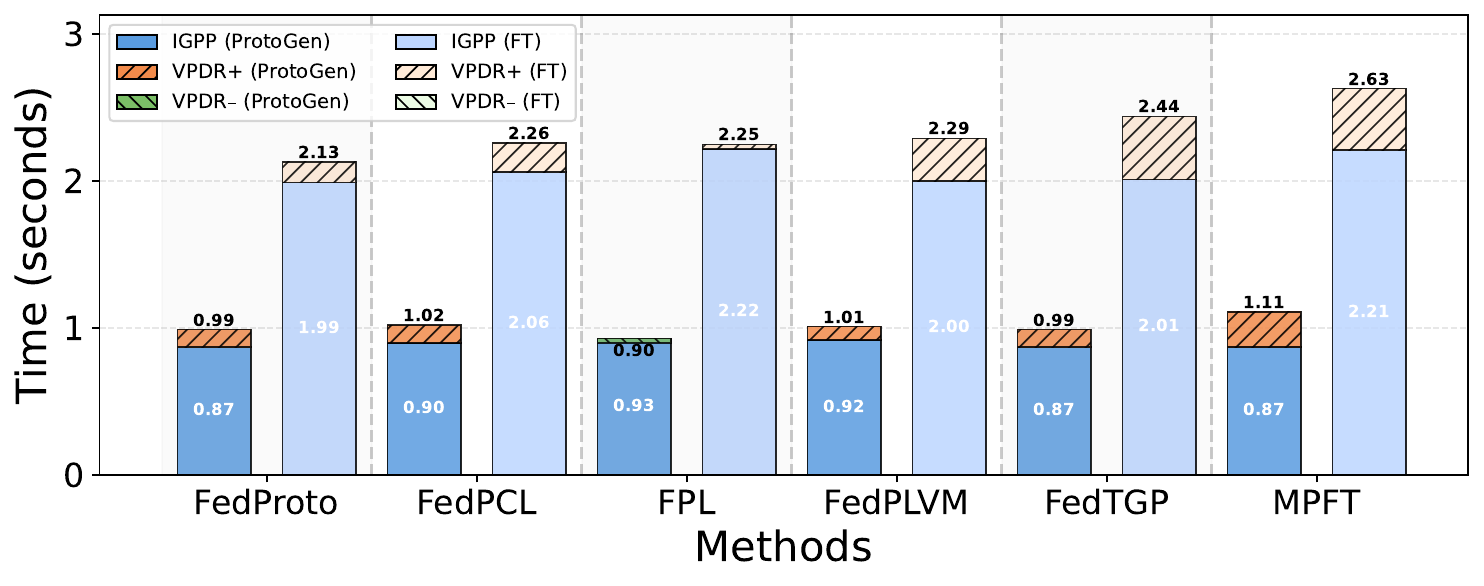}
    \caption{Office-Caltech.}
    \label{fig:time-off}
  \end{subfigure}

  \begin{subfigure}[b]{0.88\linewidth}
    \includegraphics[width=\linewidth]{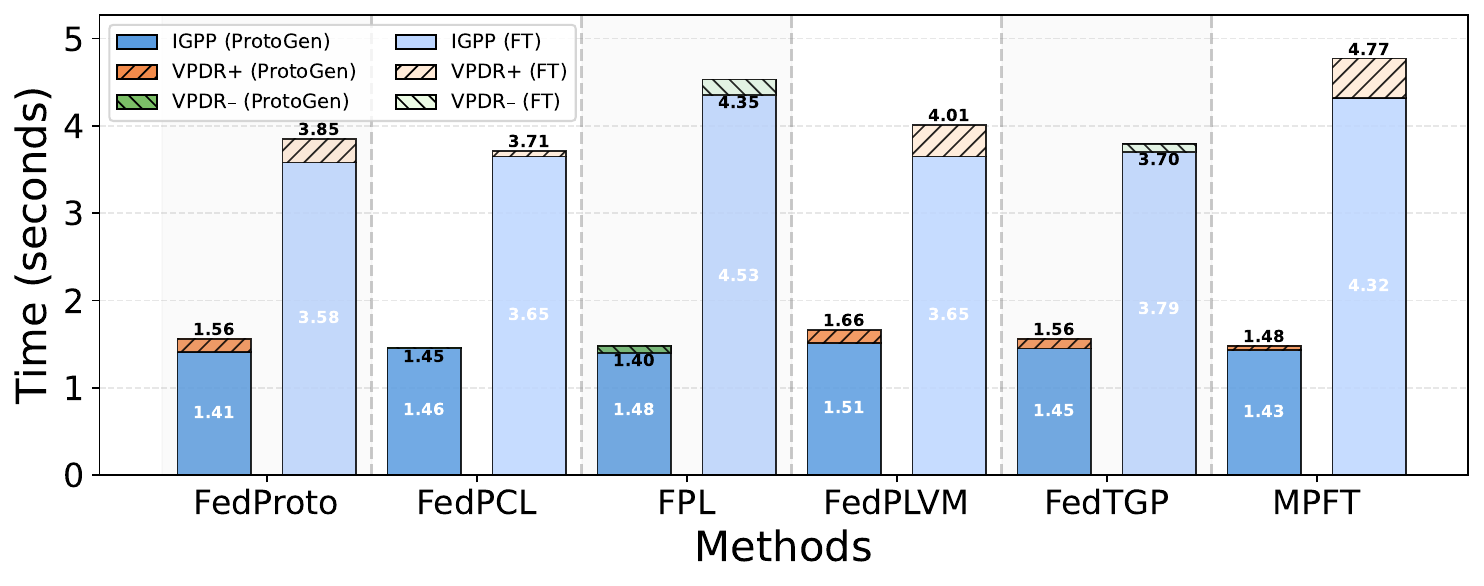}
    \caption{PACS.}
    \label{fig:time-pacs}
  \end{subfigure}
  
  \vspace{-2mm}
  \caption{\textbf{Per-client average time (s)} per communication round. }
  \label{fig:time}
\end{figure}

\begin{table}[t]
\centering
\caption{Runtime (ms) of FedProto with VPDR on Office-Caltech.}
\label{tab:time_overhead}
\vspace{-2mm}
\setlength{\tabcolsep}{0.6pt}
\fontsize{8}{7.66}\selectfont
\renewcommand{\arraystretch}{1.25}
\begin{tabularx}{\columnwidth}{l || c | Y | c }
\noalign{\hrule height 0.8pt}
\rowcolor{gray!25}
Stage & Base & Extra VPDR & Overhead\\
\hline\hline
ProtoGen & 1057 & 13.22 (Stats) + 9.78 (Top-$k$) = 23.00 & \textbf{2.2\%}\\
FT & 2055 & 1.58 (Clip) + 1.59 (Tch) + 1.38 (KL) = 4.55 & \textbf{0.2\%}\\
\noalign{\hrule height 0.8pt}
\end{tabularx}
\end{table}

\section{Conclusion}
We introduce VPDR, a client-side plug-in for privacy-preserving ProtoPFL that combines VPP, which reallocates noise across feature dimensions by discriminative strength, and DCR, which enforces soft-clipping and prediction consistency via knowledge distillation. Our analysis provides LDP guarantees no weaker than those of isotropic perturbation. Across diverse ProtoPFLs and multi-domain benchmarks, VPDR delivers higher utility at fixed privacy budgets and remains robust to label skew, while preserving communication cost and incurring only modest compute overhead. 
Evaluations against two attacks show near-chance performance, indicating strong protection.

\newpage   
\section*{Acknowledgements} 
This work was supported by the Beijing Advanced Innovation Center for Future Blockchain and Privacy Computing, 
the National Natural Science Foundation of China (U25B2070, 62372493), 
the Beijing Natural Science Foundation (Z230001), 
the China Postdoctoral Fellowship Fund (2024M764092), 
and the Beihang Dare to Take Action Plan (JKF-20240773).

{
    \small
    \bibliographystyle{ieeenat_fullname}
    \bibliography{ref}
}

\clearpage
\setcounter{page}{1}

\appendix
\maketitlesupplementary
\setcounter{tocdepth}{-1}
\renewcommand{\contentsname}{}
\addtocontents{toc}{\protect\setcounter{tocdepth}{2}}
\tableofcontents

\section{Related Work} 

\subsection{Differential Privacy in Federated Learning}
In privacy-preserving Federated Learning (FL), Differential Privacy (DP) has become a widely adopted approach to safeguard sensitive information during model updates. Early methods like DP-FedAvg~\cite{geyer2017differentially} employed Gaussian mechanisms, using $\ell_2$ norm clipping and noise injection to ensure client-level DP prior to communication. This approach relies on the server aggregating updates from clients, with privacy guarantees applied at the server level, often referred to as server-side client-level DP.
Subsequent advancements, such as BLUS+LUS~\cite{Cheng2022fedblurs} and DP-FedSAM~\cite{Shi2023fedsam}, incorporated local regularization and update sparsification techniques to mitigate sensitive gradient directions, while optimizing model flatness to minimize the performance degradation induced by DP noise.
For Personalized FL (PFL), significant progress has been made in privacy-preserving methods. PPSGD~\cite{bietti2022ppsgd} introduced a personalized parameter to balance local and global models while ensuring privacy through joint DP. Similarly, FedDPA~\cite{yang2023feddpa} employs dynamic hierarchical Fisher information to select personalized parameters, and adaptively adjusts $\ell_2$ regularization to improve robustness against clipping and reduce noise interference.

However, these methods typically assume an honest server, whereas, in practice, servers are often honest-but-curious — they follow protocols but may attempt to infer sensitive information from client updates. To address this, Local Differential Privacy (LDP) ensures privacy even in the presence of semi-honest servers or external eavesdroppers. Representative methods like UDP-FL~\cite{wei2021user} add calibrated Gaussian noise to each client update, guaranteeing client-level LDP. More recent techniques, such as ACS-FL~\cite{he2023clustered}, enhance FL models under LDP by combining adaptive clipping, weight compression, and parameter shuffling. These methods consider layer-wise ranges and budget accumulation to reduce both noise and communication, improving the privacy-utility trade-off in heterogeneous IoT data. FedFR-ADP~\cite{wang2025fedfr} introduces an adaptive LDP scheme that calibrates each client's Gaussian noise based on heterogeneity measured via Earth Mover’s Distance, dynamically adjusting the privacy budget. ALDP-FL~\cite{cui2025aldp} fine-tunes per-layer clipping via moving-average norms and injects bounded layer-wise LDP noise, improving utility while mitigating reconstruction attacks.

Despite these advancements, existing LDP methods primarily focus on privacy for gradient-based updates or model parameters, with limited attention to the privacy risks associated with prototype frameworks.

\subsection{Prototype-based Personalized FL}  
Prototype-based Personalized FL (ProtoPFL) leverages class prototypes to model multi-domain feature distribution shifts, enhancing personalization through shared and local prototypes. Early works like FedProto~\cite{FedProto2022} introduced prototype exchange to mitigate gradient misalignment in heterogeneous data, while FedPCL~\cite{fedpcl2022} used prototype-level contrastive learning to align local prototypes without sharing model parameters. FPL~\cite{fplhuang2023} introduced cluster-based and unbiased prototypes to improve model stability and discriminability across domains, and FedGMKD~\cite{zhang2024fedgmkd} incorporated knowledge distillation and clustering to improve robustness to heterogeneous data without relying on public datasets.
FedPLVM~\cite{fedplvm2024} proposed a dual-level clustering method to reduce communication costs while improving generalization with an $\alpha$-sparsity loss function. More recent works have incorporated server-side training, such as FedTGP~\cite{fedtgp2024}, which introduced trainable global prototypes and adaptive-margin contrastive learning to enhance inter-class distance and semantic consistency, and FedKTL~\cite{fedktl2024}, which uses a pre-trained server model to generate class prototype pairs for multi-domain semantic alignment. MPFT~\cite{zhang2025mpft} uploads a small number of client-generated prototypes for global model fine-tuning, enhancing multi-domain adaptation. These advancements have shifted ProtoPFL from static class-center sharing to dynamic semantic alignment and knowledge fusion, improving feature space alignment and model generalization across clients.

However, because class prototypes are directly derived from client-specific sensitive data, sharing them can inadvertently expose private semantic information, especially in highly heterogeneous and sensitive tasks. While some defenses have been proposed~\cite{fedpcl2022,fedplvm2024,zhang2025mpft}, such as injecting isotropic Gaussian noise into prototypes, these methods often lack formal privacy guarantees or fail to account for per-example clipping or local DP for each client. Our approach addresses this by providing formal LDP guarantees within ProtoPFL, ensuring privacy during prototype exchange while maintaining the utility of shared prototypes. 

\section{Notations} \label{app:notaions} 
Table~\ref{tab:notation} presents the formal notation used in this paper.
\label{app:notations}

\begin{table}[t]
\centering
\caption{Summary of notation.}
\vspace{-2mm}
\label{tab:notation}
\fontsize{8}{7.66}\selectfont
\setlength{\tabcolsep}{1pt}  
\renewcommand{\arraystretch}{1.25} 
\begin{tabularx}{\columnwidth}{l||L}
\noalign{\hrule height 0.8pt}
\rowcolor{gray!30}
\textbf{Symbol} & \textbf{Description} \\
\hline\hline
$C,\,M;\,c,\,m$ & Number of classes and clients; class and client indices \\
$T,\,E,\,B$ & Communication rounds, local epochs, and batch size \\
$\mathcal{D}_m,\,\mathcal{Z}_m$ & Client-$m$ local dataset and embedding set\\
$n_m,\,n_m^c,\,n_m^{c,k}$ & Client-$m$ total, per-class, and per-cluster sample counts \\
$d$ & Embedding dimension \\
\hline
$g,\,h_m,\,f_m$ & Backbone, client-$m$ encoder, and classifier \\ 
$\mathbf{z},\,\overline{\mathbf{z}},\,\hat{\mathbf{z}}$ & Original, $\ell_2$-clipped, and soft-clipped features \\
$\mathbf{p}_m^{c,k}$ & Prototype of class $c$ on client $m$ (cluster $k$) \\
$\mathcal{P}_m,\,\mathcal{P}_g$ & Client-$m$ local and global prototype set \\
\hline
$S_j$ & Discriminability score of coordinate $j$ \\
$H$ & Clipping cap applied to scores $S_j$ \\
$\rho$ & Fraction of dimensions kept as discriminative \\
$\mathcal{I}_A,\,\mathcal{I}_B$ & Discriminative / residual index sets \\
$d_A,\,d_B$ & Dimensions of discriminative and residual subspaces \\
$R$ & Global clipping threshold \\ 
$R_A,\,R_B$ & Groupwise clipping thresholds on $\mathcal{I}_A$ and $\mathcal{I}_B$ \\
\hline
$\epsilon,\,\delta,\,\Delta_f$ & LDP budget parameters and $\ell_2$-sensitivity of $f$ \\
$\epsilon_1,\,\epsilon_2$ & Privacy budgets for partition and prototype release \\
$r$ & Privacy split ratio ($\epsilon_1 = r\epsilon$, $\epsilon_2 = (1-r)\epsilon$) \\ 
$\Delta,\,\Delta_A,\,\Delta_B$ & Isotropic and groupwise prototype sensitivities \\
$\lambda$ & Noise scale for oneshot Laplace Top-$k$\\
$\sigma_{\mathrm{iso}},\,\sigma_{\mathrm{ref}}$ & Noise multipliers of IGPP and reference mechanism \\
$\sigma_A,\,\sigma_B$ & Groupwise noise multipliers in VPP \\
\hline
$\gamma,\,\beta$ & Soft-clipping strength and EMA momentum \\ 
$y_m^t,\,y_m^s$ & Client-$m$ Teacher and student soft targets  \\
$\tau,\,\lambda_1$ & Distillation temperature and KD weight \\
$\mathcal{L}_{\mathrm{BASE}},\,\mathcal{L}_{\mathrm{KD}}$ & Base ProtoPFL loss, and distillation loss\\
\noalign{\hrule height 0.8pt}
\end{tabularx}
\end{table}

\section{Privacy Budget Allocation and Calibration}
\label{app:Privacy Allocation}

\begin{figure}[hbpt]
    \centering
    \includegraphics[width=0.9\linewidth]{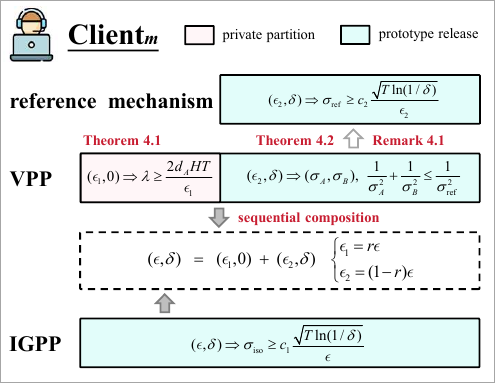}
    \vspace{-2mm}
    \caption{\textbf{Privacy allocation and calibration}. Schematic view of how IGPP and VPP use the per-client budget $(\epsilon,\delta)$ and how the reference isotropic mechanism is employed to calibrate VPP.}
    \label{fig:privacy_allocation}
\end{figure}

Figure~\ref{fig:privacy_allocation} compares how the per-client local DP budget $(\epsilon,\delta)$ is allocated in IGPP and VPP. IGPP spends the entire budget on an isotropic Gaussian prototype mechanism, choosing $\sigma_{\mathrm{iso}}$ so that the resulting $T$-round mechanism is $(\epsilon,\delta)$–LDP for each client (Theorem~\ref{thm:wei-noise}).
In contrast, VPP decomposes the local budget as $(\epsilon,\delta) = (\epsilon_1,0) + (\epsilon_2,\delta)$ with $\epsilon_1 = r\epsilon$ and $\epsilon_2 = (1-r)\epsilon$. The pure-DP share $(\epsilon_1,0)$ is used for the private discriminative partition on scores ${S_j}$ over $T$ rounds (Theorem~\ref{thm:vpp-partition-pure}), while $(\epsilon_2,\delta)$ is reserved for T rounds of prototype release.
To analyze the release in VPP, we introduce a conceptual \emph{reference} isotropic Gaussian mechanism from the same family as IGPP. IGPP corresponds to using this family under the full budget $(\epsilon,\delta)$ with multiplier $\sigma_{\mathrm{iso}}$, whereas the reference mechanism is calibrated only to $(\epsilon_2,\delta)$ with multiplier $\sigma_{\mathrm{ref}}$, so that it achieves $(\epsilon_2,\delta)$–LDP over $T$ rounds (Theorem~\ref{thm:wei-noise}). VPP then replaces this reference isotropic release with an anisotropic Gaussian mechanism with groupwise multipliers $(\sigma_A,\sigma_B)$. Under the calibration condition~\eqref{eq:calibration}, Theorem~\ref{thm:vpp-Prototype-Release} shows that its privacy loss is no larger than that of the reference mechanism, and hence the VPP release also satisfies $(\epsilon_2,\delta)$–LDP. By sequentially composing the partition step with budget $(\epsilon_1,0)$ and the release step with budget $(\epsilon_2,\delta)$, VPP attains an overall $(\epsilon,\delta)$–LDP guarantee. See details in Section~\ref{app:proof of vpp-Prototype-Release}

\section{Details for Privacy Analysis}
\label{app:More Proof of Privacy Analysis}
All mechanisms are analyzed as acting locally on a single client's dataset under the per-example adjacency relation $\mathcal{D}\sim\mathcal{D}'$ defined in Section~\ref{subsec:LDP}. Consequently, the resulting $(\epsilon,\delta)$ guarantees represent client-side Local DP guarantees.
 
\subsection{Proof of Theorem~\ref{thm:vpp-partition-pure}}
\PrivacyForPrivatePartition*
\begin{proof}
Clipping yields per-coordinate $\ell_1$ sensitivity
\[
\Delta_f \le \sup_{\mathcal{D}\sim\mathcal{D}'} \bigl|f_j(\mathcal{D}) - f_j(\mathcal{D}')\bigr| \;\le\; H.
\]
By Theorem~\ref{thm:pure-dp}, the oneshot Laplace Top-$k$ mechanism run on a client in each round is
$(\epsilon_1/T,0)$–DP provided
\[
  \lambda \;\ge\; \frac{2 d_A H T}{\epsilon_1},
\]
and the index partition is obtained by post-processing. Composing over $T$ rounds then yields
$(\epsilon_1,0)$–LDP for client $m$ in FL, as claimed.
\end{proof}

\subsection{Proof of Theorem~\ref{thm:vpp-Prototype-Release}} \label{app:proof of vpp-Prototype-Release}
We adopt the R\'enyi (RDP) framework~\cite{mironov2017renyi} and analyze each client's local mechanism under the per-example adjacency $\mathcal{D}\sim\mathcal{D}'$ defined in Section~\ref{subsec:LDP}. Applying the RDP-to-DP conversion then yields LDP guarantees~\cite{wei2021user,tran2025privacy}. 
The relevant definitions are recalled below.

\begin{definition}[R\'enyi Divergence \cite{van2014renyi}]
Given two probability distributions \( P \) and \( Q \), the R\'enyi divergence of order \( \alpha > 1 \) between \( P \) and \( Q \) is defined as
\[D_{\alpha}(P \| Q)
= \frac{1}{\alpha - 1} \log \mathbb{E}_{x \sim Q} \!\left[ \Bigl(\frac{P(x)}{Q(x)}\Bigr)^{\alpha} \right].
\]
\end{definition}

\begin{definition}[RDP~\cite{mironov2017renyi}]
A randomized mechanism \( \mathcal{M}: \mathcal{X} \to \mathcal{R} \) is said to satisfy \( (\alpha, \epsilon(\alpha)) \)-RDP if for all adjacent datasets \(\mathcal{D}\sim \mathcal{D'}\), 
\[
D_{\alpha}(\mathcal{M}(\mathcal{D}) \| \mathcal{M}(\mathcal{D'})) \leq \epsilon(\alpha).
\]
\end{definition}

\begin{lemma}[(\(\alpha, \epsilon(\alpha)\))-RDP to \((\epsilon, \delta)\)-LDP~\cite{mironov2017renyi}]
\label{lem:rdptodp}
If a mechanism \(\mathcal{M}\) satisfies \((\alpha, \epsilon(\alpha))\)-RDP under the per-example adjacency $\mathcal{D}\sim\mathcal{D}'$, then for any \(\delta \in (0,1)\) it also satisfies \((\epsilon, \delta)\)-LDP with
\(\epsilon= \epsilon(\alpha) + \log(1/\delta)/(\alpha - 1) \).
\end{lemma}

\begin{restatable}{lemma}{RDPcovariance} 
\label{lem:RDP-covariance} 
Let $\alpha>1$ and let $P=\mathcal N(\mu,\Sigma)$ and $Q=\mathcal N(\mu',\Sigma)$
be two $d$-dimensional Gaussians with the same positive–definite covariance $\Sigma\in\mathbb{R}^{d\times d}$. Then 
\[
D_\alpha(P\|Q)=\frac{\alpha}{2}(\mu-\mu')^\top\Sigma^{-1}(\mu-\mu'). 
\]
\end{restatable} 

\begin{corollary}
\label{cor:isotropic-covariance}
For isotropic covariance $\Sigma=\sigma^{2}\mathbf I_d$, 
\[
D_\alpha\bigl(\mathcal N(\mu,\sigma^2\mathbf I_d)\,\|\,\mathcal N(\mu',\sigma^2\mathbf I_d)\bigr)
= \frac{\alpha}{2\sigma^2}\,\|\mu-\mu'\|_2^{2}.
\]
\end{corollary}

\PrivacyforPrototypeRelease*
\begin{proof}
We view the isotropic Gaussian mechanism with sensitivity $\Delta$ that achieves $(\epsilon_2,\delta)$--LDP as a \emph{reference mechanism}. 
Let $v = f(\mathcal{D})-f(\mathcal{D}')$ denote the difference vector for adjacent datasets $\mathcal{D}\sim\mathcal{D}'$. Based on our partition, we decompose this vector into $v=[v_A;v_B]$. 
The groupwise clipping bound the sensitivities such that $\|v_A\|_2 \le \Delta_A$ and $\|v_B\|_2 \le \Delta_B$, satisfying $\Delta^2 = \Delta_A^2 + \Delta_B^2$.

\noindent \underline{\textit{Reference (isotropic) release.}}
For covariance $\Sigma_{\mathrm{ref}}=(\sigma_{\mathrm{ref}}\Delta)^2\mathbf I_d$, Corollary~\ref{cor:isotropic-covariance} gives
\begin{equation}
\label{eq:iso-rdp}
\epsilon(\alpha)_{\mathrm{ref}}
= \frac{\alpha}{2}\frac{\|v\|_2^2}{(\sigma_{\mathrm{ref}}\Delta)^2}
\le \frac{\alpha}{2}\frac{\Delta^2}{(\sigma_{\mathrm{ref}}\Delta)^2}
= \frac{\alpha}{2\sigma_{\mathrm{ref}}^2}.
\end{equation} 

\noindent \underline{\textit{VPP (groupwise) release.}}
For covariance
$\Sigma_{\mathrm{VPP}}=\mathrm{diag}\!\bigl((\sigma_A\Delta_A)^2\mathbf I_{d_A},\,(\sigma_B\Delta_B)^2\mathbf I_{d_B}\bigr)$,
Lemma~\ref{lem:RDP-covariance} yields
\begin{align} 
\label{eq:vpp-rdp}
\epsilon(\alpha)_{\mathrm{VPP}}
= \frac{\alpha}{2} v^\top \Sigma_{\mathrm{VPP}}^{-1} v 
&\le \frac{\alpha}{2}\!\left(
\frac{\|v_A\|_2^2}{(\sigma_A\Delta_A)^2}
+\frac{\|v_B\|_2^2}{(\sigma_B\Delta_B)^2}\right) \nonumber\\ 
&\le \frac{\alpha}{2}\!\left(\frac{1}{\sigma_A^2}+\frac{1}{\sigma_B^2}\right).
\end{align}
Provided that the harmonic calibration condition in~\eqref{eq:calibration} holds, comparing~\eqref{eq:iso-rdp} and~\eqref{eq:vpp-rdp} guarantees that $\epsilon_{\mathrm{VPP}}(\alpha) \le \epsilon_{\mathrm{ref}}(\alpha)$ for all $\alpha > 1$.  
By Lemma~\ref{lem:rdptodp}, the $(\epsilon_2,\delta)$–LDP guarantee proven for the reference mechanism seamlessly transfers to VPP. Thus, the VPP prototype release strictly satisfies $(\epsilon_2,\delta)$–LDP over $T$ rounds for client $m$.  
Finally, as established in Theorem~\ref{thm:wei-noise}, the reference multiplier must satisfy $\sigma_{\mathrm{ref}} \ge c_2 \sqrt{T \ln(1/\delta)}/\epsilon_2$, completing the proof.
\end{proof} 

\vspace{0.5mm}
\noindent \textbf{Remark on Weight Allocation.}
A straightforward choice satisfying the condition in~\eqref{eq:calibration} is $\sigma_A = \sigma_{\mathrm{ref}}/\sqrt{w}$ and $\sigma_B = \sigma_{\mathrm{ref}}/\sqrt{1-w}$ for any $w\in(0,1)$, which ensures $1/\sigma_A^2 + 1/\sigma_B^2 = 1/\sigma_{\mathrm{ref}}^2$ holds with exact equality. This parameterization calibrates VPP to the precise $(\epsilon_2,\delta)$–LDP budget of the isotropic reference mechanism while strategically redistributing noise across subspaces. By fixing $w$ according to the group dimensions (as defined in~\eqref{eq:weights}), the noise allocation becomes fully adaptive to the discriminability-aware partition, automatically assigning smaller multipliers (i.e., less noise) to the more informative subspace without requiring manual hyperparameter tuning.

\section{Experimental Supplement} 
\begin{table*}[t]
\caption{\textbf{Comparison} of Average Accuracy (AVG) and Standard Deviation (STD) under domain skew at $\epsilon=0.5$. The \textbf{best} is marked.} 
\vspace{-2mm}
\centering
\small
\renewcommand{\arraystretch}{1.25}
\fontsize{8}{7.66}\selectfont 
\setlength{\tabcolsep}{1pt}
\begin{tabularx}{\linewidth}
{l l||*{4}{>{\centering\arraybackslash}X}|*{2}{>{\centering\arraybackslash}X}|
 *{4}{>{\centering\arraybackslash}X}|*{2}{>{\centering\arraybackslash}X}|
 *{4}{>{\centering\arraybackslash}X}|*{2}{>{\centering\arraybackslash}X}}
\noalign{\hrule height 0.8pt}
\rowcolor{gray!25}
\multicolumn{2}{c||}{\textbf{Methods}} & \multicolumn{6}{c|}{\textbf{Digits}} & \multicolumn{6}{c|}{\textbf{Office–Caltech}} & \multicolumn{6}{c}{\textbf{PACS}} \\
\cline{3-8} \cline{9-14} \cline{15-20} 
\rowcolor{gray!25}
Framework & +LDP & MNI. & USPS & SVHN & SYN & AVG$\uparrow$ & STD$\downarrow$
& Amz.  & Cal.  & DSR.  & Web. & AVG$\uparrow$ & STD$\downarrow$
& P.    & AP.   & Ct.  & Sk.  & AVG$\uparrow$ & STD$\downarrow$ \\ 
\hline\hline
  & +IGPP  & 98.08 & 93.57 & 89.41  & 95.25 & 94.08  & 3.63
           & 95.27 & 91.86 & 76.42 & 95.03  & 89.64  & 8.95
           & 98.10 & 92.62 & 91.24 & 79.48  & 90.36  & 7.84 \\
\multirow{-2}{*}{FedProto~\cite{FedProto2022}}
  & \textbf{+VPDR}  & 98.17  & 94.32  & 92.55  & 97.70  & \cellcolor{vpdr}{\textbf{95.69}}  & \cellcolor{vpdr}{\textbf{2.70}}
           & 96.34 & 93.75 & 83.87 & 96.61  & \cellcolor{vpdr}{\textbf{92.64}}  & \cellcolor{vpdr}{\textbf{5.99}}
           & 98.10 & 95.56 & 91.38 & 83.35  & \cellcolor{vpdr}{\textbf{92.10}}  & \cellcolor{vpdr}{\textbf{6.46}} \\
\hline
  & +IGPP  & 98.27  & 80.70  & 92.43  & 65.25  & 84.16  & 14.57
           & 95.81 & 89.29 & 61.29 & 94.64  & 85.26  & 16.23
           & 97.66 & 83.44 & 62.61 & 51.92  & 73.91  & 20.54 \\
\multirow{-2}{*}{FedPCL~\cite{fedpcl2022}}
  & \textbf{+VPDR}   & 98.47  & 85.62  & 92.89  & 75.20  & \cellcolor{vpdr}{\textbf{88.05}}  & \cellcolor{vpdr}{\textbf{10.05}}
           & 96.34 & 91.96 & 72.19 & 95.92  & \cellcolor{vpdr}{\textbf{89.10}}  & \cellcolor{vpdr}{\textbf{11.45}}
           & 97.31 & 83.62 & 77.56 & 55.92  & \cellcolor{vpdr}{\textbf{78.60}}  & \cellcolor{vpdr}{\textbf{17.23}} \\
\hline
  & +IGPP  & 98.09  & 92.52  & 91.49  & 96.45  & 94.64  & 3.14
           & 94.53 & 91.86 & 81.65 & 95.77  & 90.95  & 6.41
           & 98.20 & 92.91 & 90.81 & 79.75  & 90.42  & 7.76 \\
\multirow{-2}{*}{FPL~\cite{fplhuang2023}}
  & \textbf{+VPDR}  & 98.24  & 94.22  & 93.03  & 97.30  & \cellcolor{vpdr}{\textbf{95.70}}  & \cellcolor{vpdr}{\textbf{2.47}}
           & 96.08 & 92.41 & 83.87 & 98.31  & \cellcolor{vpdr}{\textbf{92.67}}  & \cellcolor{vpdr}{\textbf{6.35}}
           & 98.40 & 96.56 & 91.74 & 81.64  & \cellcolor{vpdr}{\textbf{92.09}}  & \cellcolor{vpdr}{\textbf{7.51}} \\
\hline
  & +IGPP  & 97.79  & 89.94  & 92.07  & 89.20  & 92.25  & 3.89
           & 95.29 & 88.84 & 79.42 & 96.42  & 89.99  & 7.80
           & 97.90 & 90.95 & 82.69 & 65.99  & 84.38  & 13.75 \\
\multirow{-2}{*}{FedPLVM~\cite{fedplvm2024}}
  & \textbf{+VPDR}  & 98.14  & 91.59  & 93.17  & 91.45  & \cellcolor{vpdr}{\textbf{93.59}}  & \cellcolor{vpdr}{\textbf{3.13}}
           & 96.34 & 92.86 & 81.65 & 98.04  & \cellcolor{vpdr}{\textbf{92.22}}  & \cellcolor{vpdr}{\textbf{7.37}}
           & 98.80 & 92.18 & 86.54 & 68.84  & \cellcolor{vpdr}{\textbf{86.59}}  & \cellcolor{vpdr}{\textbf{12.85}} \\
\hline
  & +IGPP  & 97.59  & 93.82  & 90.61  & 96.90  & 94.73  & 3.20
           & 95.81 & 91.96 & 83.87 & 96.61  & 92.06  & \textbf{5.83}
           & 98.20 & 95.11 & 90.17 & 81.91  & 91.35  & \textbf{7.11} \\
\multirow{-2}{*}{FedTGP~\cite{fedtgp2024}}
  & \textbf{+VPDR}  & 97.86  & 95.22  & 92.42  & 98.40  & \cellcolor{vpdr}{\textbf{95.97}}  & \cellcolor{vpdr}{\textbf{2.75}}
           & 96.86 & 95.09 & 84.15 & 100  & \cellcolor{vpdr}{\textbf{94.03}}  & \cellcolor{vpdr}{6.89}
           & 99.10 & 97.56 & 92.95 & 83.06  & \cellcolor{vpdr}{\textbf{93.17}}  & \cellcolor{vpdr}{7.23} \\
\hline
  & +IGPP  & 97.56  & 93.92  & 92.83  & 97.60  & 95.48  & \textbf{2.47}
           & 91.10 & 91.07 & 80.65 & 94.92  & 89.44  & \textbf{6.13}
           & 99.20 & 95.84 & 92.31 & 82.68  & 92.51  & 7.13 \\
\multirow{-2}{*}{MPFT~\cite{zhang2025mpft}}
  & \textbf{+VPDR}  & 98.55  & 94.53  & 93.50  & 98.80  & \cellcolor{vpdr}{\textbf{96.34}}  & \cellcolor{vpdr}{2.73}
           & 97.38 & 91.07 & 83.87 & 98.31  & \cellcolor{vpdr}{\textbf{92.66}}  & \cellcolor{vpdr}{6.68}
           & 99.20 & 96.09 & 93.52 & 83.06  & \cellcolor{vpdr}{\textbf{92.97}}  & \cellcolor{vpdr}{\textbf{7.00}} \\
\noalign{\hrule height 0.8pt}
\end{tabularx}
\label{tab:all_datasets_eps05}
\end{table*}

\subsection{Details of Datasets}\label{app:dataset}
We evaluate our method on three multi-domain benchmarks: Digits, Office–Caltech, and PACS. Below we provide a detailed description of each benchmark.
\begin{itemize}[itemsep=1pt, topsep=2pt, leftmargin=10pt]
\item \textbf{Digits}~\cite{hull1994database, lecun1998gradient, netzer2011reading, roy2018effects} consists of several handwritten digit recognition datasets, including MNIST, USPS, SVHN, and synthetic digits. These datasets cover a variety of digit recognition tasks, providing a testbed for domain adaptation and skew studies.
\item \textbf{Office-Caltech}~\cite{gong2012geodesic} is a benchmark for object recognition and consists of four domains: Caltech-256, Amazon, DSLR, and Webcam. The dataset is commonly used to evaluate domain adaptation and generalization across real-world and synthetic image domains.
\item \textbf{PACS}~\cite{li2017deeper} is a cross-domain dataset for domain generalization. It includes four domains: Photo, Art Painting, Cartoon, and Sketch. PACS is widely used to assess model robustness and generalization in visual recognition across multiple domains.
\end{itemize}
Additionally, we use CIFAR-10 in the single-domain setting to study the scalability under label skew. 
\begin{itemize}[itemsep=1pt, topsep=2pt, leftmargin=10pt]
\item \textbf{CIFAR-10}~\cite{krizhevsky2009learning} consists of 60,000 32x32 color images across 10 classes, with 50,000 training images and 10,000 test images. In this setting, we simulate label skew by partitioning the dataset among clients, where each client receives a non-uniform distribution of class labels. Specifically, we achieve label skew by applying a Dirichlet distribution \( \text{Dir}(\alpha) \) with \( \alpha > 0 \) to sample the class proportions for each client. The parameter \( \alpha \) controls the degree of skew, with smaller values leading to more imbalanced local distributions.
\end{itemize}

\subsection{Hyperparameters of ProtoPFLs}
\label{app:Beseline Hyperparameters}
Table~\ref{tab:hyperparams} lists hyperparameters for ProtoPFL frameworks. Hyperparameters in different methodologies may share the same notation but represent distinct meanings.

\begin{table}[hbpt]
\centering
\caption{\textbf{Hyperparameters chosen} for different methods.}
\label{tab:hyperparams} 
\vspace{-2mm}
\setlength{\tabcolsep}{4pt}
\fontsize{8}{7.66}\selectfont
\renewcommand{\arraystretch}{1.25}
\begin{tabularx}{\columnwidth}{l || l |Y}
\noalign{\hrule height 0.8pt}
\rowcolor{gray!25} 
Methods & Hyperparameters & Value \\
\hline\hline
FedProto & Proximal weight $\lambda$ & 0.1 \\
\hline
\multirow{1}{*}{FedPCL} & Contrastive temperature $\tau$ & 0.07 \\ 
\hline
\multirow{1}{*}{FPL} & Contrastive temperature $\tau$ & 0.07 \\ 
\hline
\multirow{3}{*}{FedPLVM} 
 & $\alpha$-sparsity $\alpha$ & 0.25 \\
 & Contrastive temperature $\tau$ & 0.07 \\
 & Proto loss weight $\lambda$ & 0.5 \\
\hline
\multirow{3}{*}{FedTGP} 
 & Prototype regularization weight $\lambda$ & 1.0 \\
 & Server training rounds & 100.0 \\
 & margin learning threshold & 100.0 \\
\hline
\multirow{4}{*}{MPFT} 
 & Distill temperature $T$ & 4 \\
 & Distill weight $\beta$ & 0.3 \\
 & Clustering sampling rate $r$ & 0.1 \\
 & Server learning rate $\eta_s$ & 0.001 \\
 & Server training rounds & 20 \\
\noalign{\hrule height 0.8pt}
\end{tabularx}
\end{table}

\subsection{More Results for Section~\ref{sec:Performance Comparison}}
\label{app:More Results of Performance Comparison} Table~\ref{tab:all_datasets_eps05} demonstrates that VPDR consistently outperforms IGPP in terms of average accuracy (AVG) across all ProtoPFL frameworks and benchmarks at $\epsilon=0.5$. Furthermore, the qualitative t-SNE visualizations for the Digits dataset at $\epsilon=1.0$ (Figure~\ref{fig:tsne_digits}) reveal that VPDR produces tighter and more distinct clusters as training progresses into the later rounds.

\begin{figure}[t]
  \centering
  \begin{minipage}[t]{0.24\linewidth}
    \centering
    {\fontsize{6}{7}\selectfont\itshape MNIST: 95.52 @T=1}
    \includegraphics[width=\linewidth]{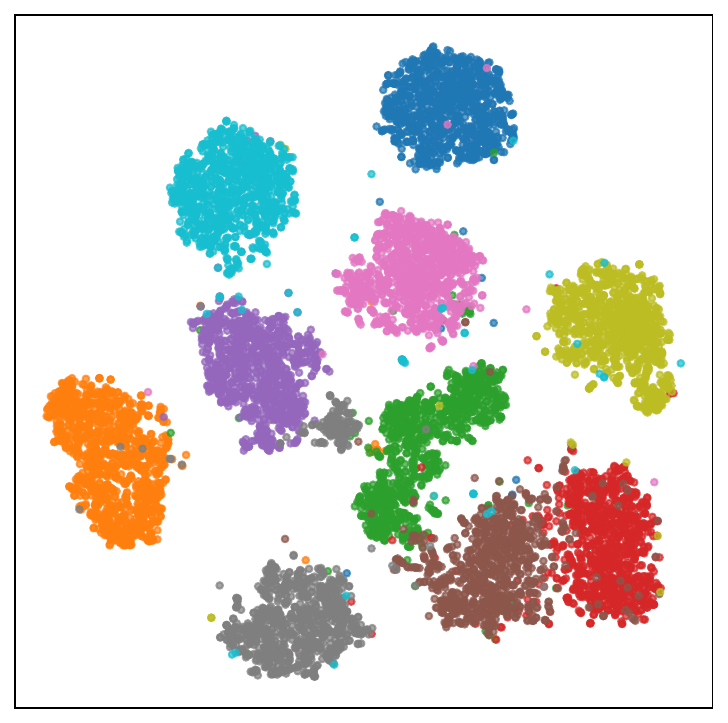}
  \end{minipage}\hfill
  \begin{minipage}[t]{0.24\linewidth}
    \centering
    {\fontsize{6}{7}\selectfont\itshape USPS: 82.96 @T=1}
    \includegraphics[width=\linewidth]{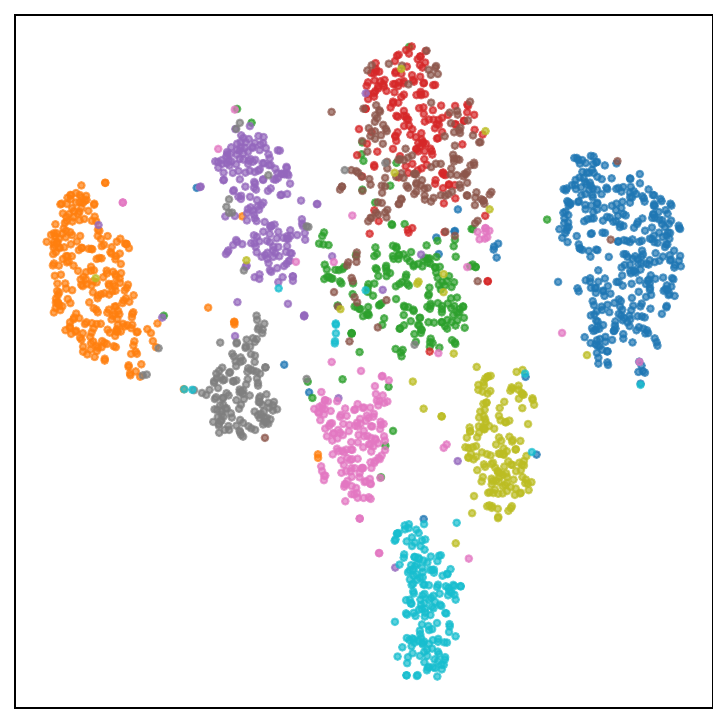}
  \end{minipage}\hfill
  \begin{minipage}[t]{0.24\linewidth}
    \centering
    {\fontsize{6}{7}\selectfont\itshape SVHN: 85.79 @T=1}
    \includegraphics[width=\linewidth]{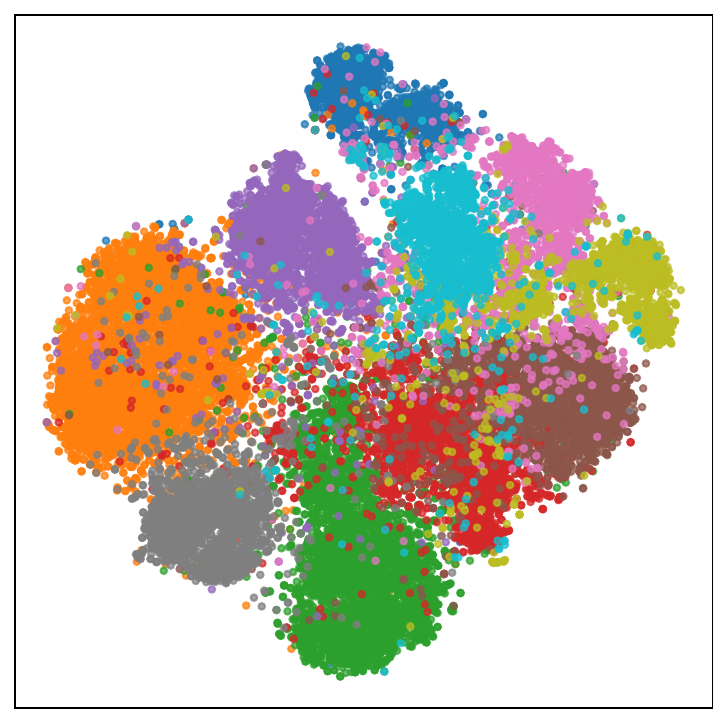}
  \end{minipage}\hfill
  \begin{minipage}[t]{0.24\linewidth}
    \centering
    {\fontsize{6}{7}\selectfont\itshape SYN: 34.55 @T=1}
    \includegraphics[width=\linewidth]{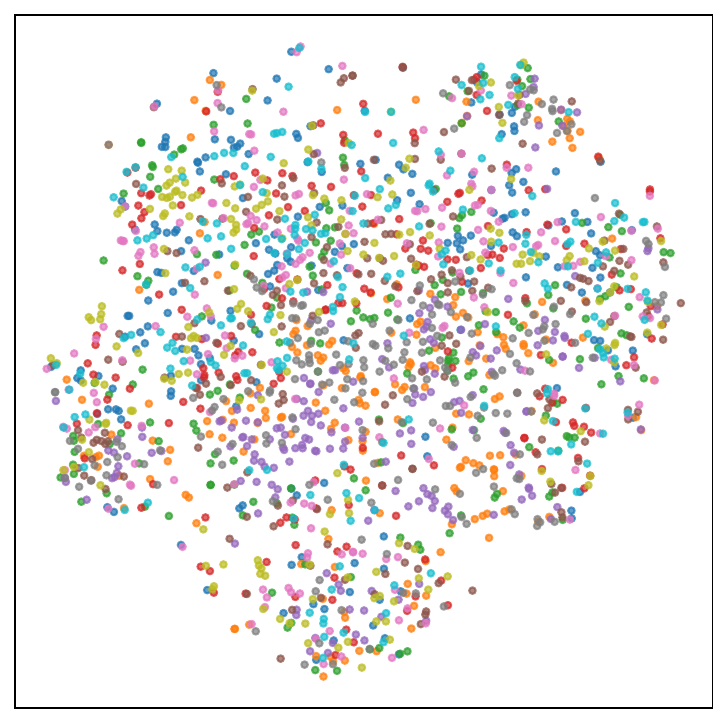}
  \end{minipage}
  
  \vspace{1ex}  
  \begin{minipage}[t]{0.24\linewidth}
    \centering
    {\fontsize{6}{7}\selectfont\itshape MNIST: 98.31 @T=20}
    \includegraphics[width=\linewidth]{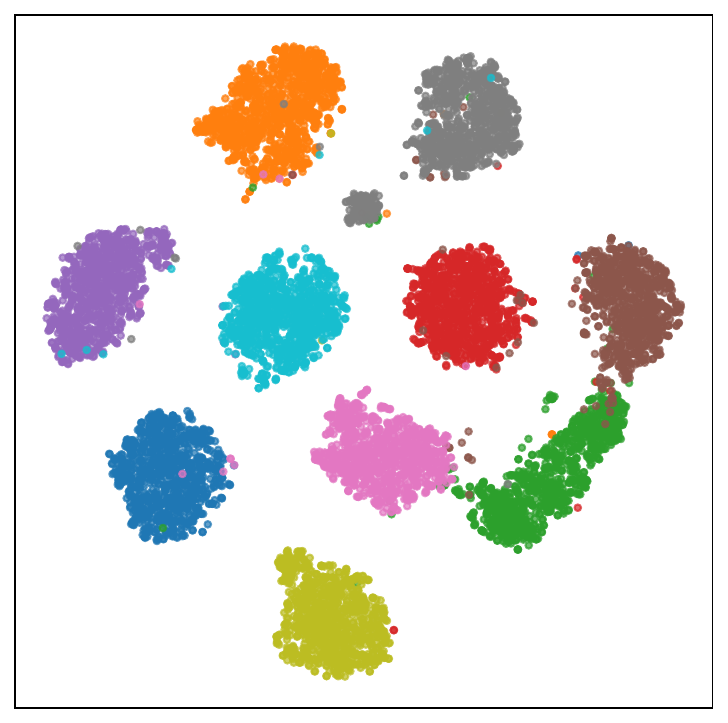}
  \end{minipage}\hfill
  \begin{minipage}[t]{0.24\linewidth}
    \centering
    {\fontsize{6}{7}\selectfont\itshape USPS: 94.87 @T=20}
    \includegraphics[width=\linewidth]{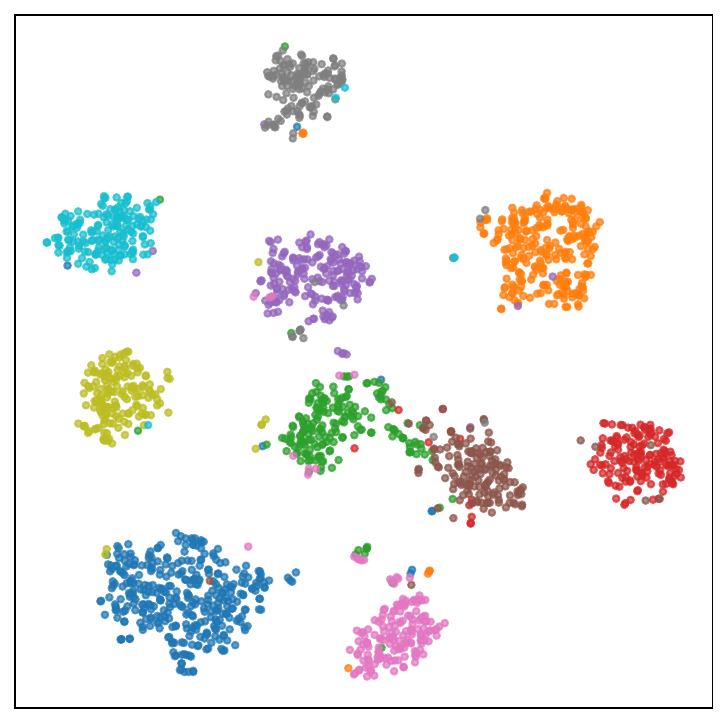}
  \end{minipage}\hfill
  \begin{minipage}[t]{0.24\linewidth}
    \centering
    {\fontsize{6}{7}\selectfont\itshape SVHN: 92.33 @T=20}
    \includegraphics[width=\linewidth]{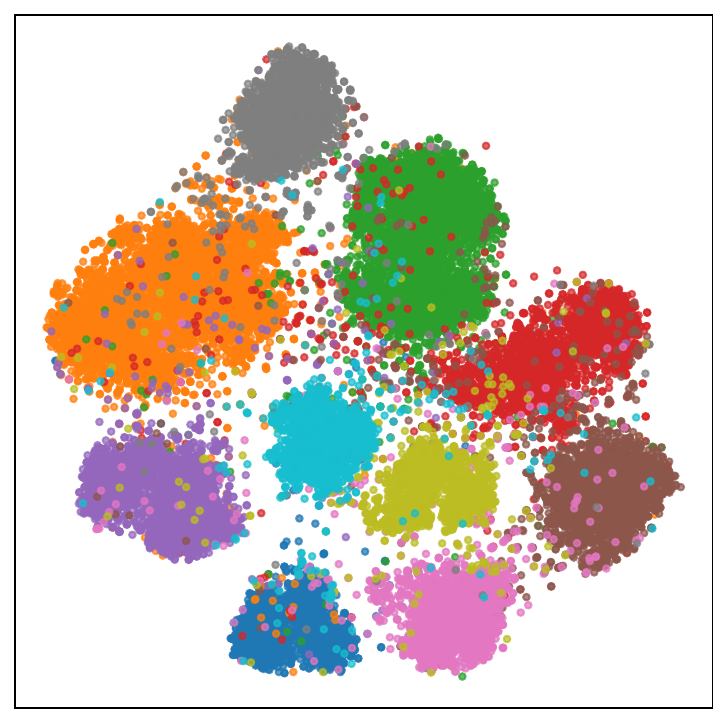}
  \end{minipage}\hfill
  \begin{minipage}[t]{0.24\linewidth}
    \centering
    {\fontsize{6}{7}\selectfont\itshape SYN: 98.70 @T=20}
    \includegraphics[width=\linewidth]{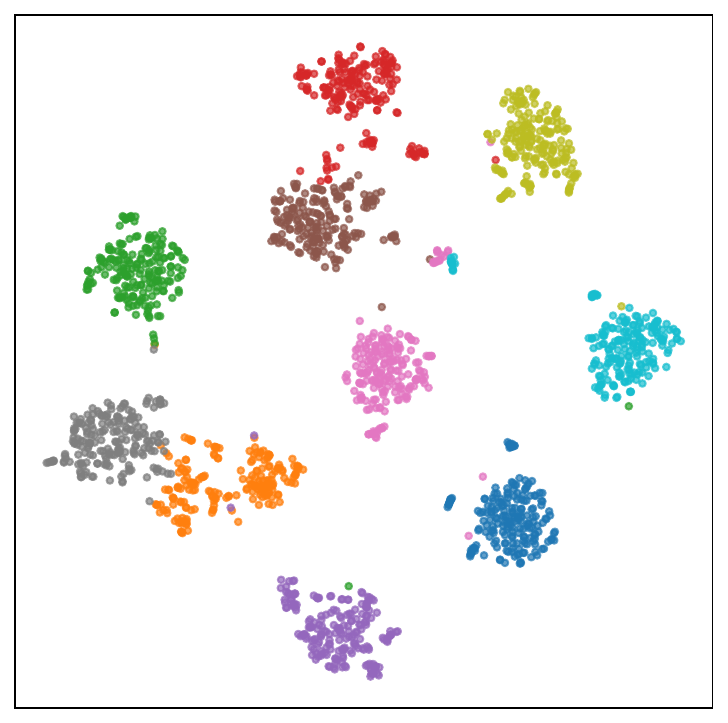}
  \end{minipage}
   \caption{\textbf{T-SNE visualization} of FedProto~\cite{FedProto2022} framework with VPDR on Digits at $T=1$ (top row) and $T=20$ (bottom row). Test accuracy (\%) is shown above each subplot.}
  \label{fig:tsne_digits}
\end{figure} 

\subsection{Model with Data Heterogeneity}
\label{app:model-data-skew}
To evaluate VPDR under highly challenging conditions, we couple architectural diversity with two forms of data skew: (i) domain shift across multi-domain benchmarks, and (ii) label imbalance on CIFAR-10. We instantiate two four-way heterogeneous model families~\cite{fedtgp2024,fedktl2024}: 
\begin{itemize}[itemsep=1pt, topsep=2pt, leftmargin=10pt]
\item \textbf{Heterogeneous Feature Extractors (HtFE$_4$).} Each client uses a different backbone architecture: ResNet-18/34, ViT-Tiny/Small~\cite{dosovitskiy2021vit}. All backbones feed into a shared 64-channel adapter and global- or patch-mean pooling to produce 512-dimensional features, so that prototypes lie in a common embedding space while encoder family, depth, and capacity vary across clients.  
\item \textbf{Heterogeneous Classifiers (HtC$_4$).} All clients share a ResNet-18 backbone and the same 512-dimensional representation, but attach classifier heads with increasing depth and width (from a single 512-d linear layer to 2–3 layer MLPs with hidden sizes up to $1024\to 512$). 

\end{itemize}

Under these two model families, we evaluate two data regimes. For domain skew, we follow the main setup on Digits, Office–Caltech (Office), and PACS and report average test accuracy (\%) at $\epsilon=1.0$ (Table~\ref{tab:Model_heterogeneity_domain_skew}). For label skew, we use Dirichlet splits of CIFAR-10 with $\alpha\in\{0.1,0.5\}$ under the same privacy budget (Table~\ref{tab:Model_heterogeneity_label_skew}). In all cases we run FedProto~\cite{FedProto2022} and FedTGP~\cite{fedtgp2024} with IGPP and our VPDR to test whether the gains of VPDR persist under joint model–data heterogeneity. 
Across all configurations in Tables~\ref{tab:Model_heterogeneity_domain_skew} and~\ref{tab:Model_heterogeneity_label_skew}, VPDR consistently matches or improves upon IGPP for both FedProto and FedTGP. Under domain skew, gains on Digits are modest due to saturated accuracy, but VPDR yields clearly larger improvements on Office–Caltech and PACS, especially under HtC$_4$. For label skew, VPDR again brings steady gains under HtFE$_4$ and more pronounced gains under HtC$_4$, and never hurts performance even in the highly skewed case with $\alpha=0.1$. Overall, VPDR remains robust and often more effective than IGPP when both model architectures and data distributions are heterogeneous.

\begin{table}[t]
\centering
\caption{\textbf{Model heterogeneity} with domain skew.}
\label{tab:Model_heterogeneity_domain_skew}
\vspace{-2mm}
\setlength{\tabcolsep}{3pt}
\fontsize{8}{7.8}\selectfont
\renewcommand{\arraystretch}{1.25}
\begin{tabularx}{\columnwidth}{ll||YYY|YYY}
\noalign{\hrule height 0.8pt}
\rowcolor{gray!25}
\multicolumn{2}{c||}{} 
& \multicolumn{3}{c|}{HtFE$_4$} 
& \multicolumn{3}{c}{HtC$_4$}\\
\cline{3-8}
\rowcolor{gray!25}
\multicolumn{2}{c||}{Methods}
& Digits & Office & PACS
& Digits & Office & PACS\\
\hline\hline
\multirow{2}{*}{FedProto~\cite{FedProto2022}}
& +IGPP        & 94.93 & 87.26 & 88.81 & 91.54 & 81.90 & 88.54\\
& \textbf{+VPDR} & 95.24 & 88.59 & 90.46 & 93.36 & 84.44 & 91.84\\
\hline
\multirow{2}{*}{FedTGP~\cite{fedtgp2024}}
& +IGPP        & 95.11 & 85.07 & 89.37 & 92.36 & 72.67 & 89.78\\
& \textbf{+VPDR} & 95.31 & 87.58 & 89.97 & 93.52 & 85.25 & 91.56\\
\noalign{\hrule height 0.8pt}
\end{tabularx}
\end{table}

\begin{table}[t]
\centering
\caption{\textbf{Model heterogeneity} with label skew.}
\label{tab:Model_heterogeneity_label_skew}
\vspace{-2mm}
\setlength{\tabcolsep}{2pt}
\fontsize{8}{7.66}\selectfont
\renewcommand{\arraystretch}{1.25}
\begin{tabularx}{\columnwidth}{l l || YY|YY}
\noalign{\hrule height 0.8pt}
\rowcolor{gray!25}
\multicolumn{2}{c||}{} 
& \multicolumn{2}{c|}{HtFE$_4$} 
& \multicolumn{2}{c}{HtC$_4$}\\
\cline{3-6}
\rowcolor{gray!25}
\multicolumn{2}{c||}{Methods} 
& $\alpha=0.1$ & $\alpha=0.5$ 
& $\alpha=0.1$ & $\alpha=0.5$\\
\hline\hline
\multirow{2}{*}{FedProto~\cite{FedProto2022}}
  & +IGPP        & 37.42 & 76.32 & 32.66 & 69.67 \\
  & \textbf{+VPDR} & 37.94 & 76.90 & 34.45 & 71.27 \\
\hline
\multirow{2}{*}{FedTGP~\cite{fedtgp2024}}
  & +IGPP        & 35.16 & 75.22 & 31.38 & 66.60 \\
  & \textbf{+VPDR} & 36.91 & 76.34 & 34.13 & 70.47 \\
\noalign{\hrule height 0.8pt}
\end{tabularx}
\end{table}

\subsection{Generalization to Text Modality} 
\label{app:generalization-text}
To verify that our framework generalizes beyond computer vision tasks, we evaluate VPDR on the AG News dataset, a 4-class topic classification benchmark containing 120,000 training articles (World, Sports, Business, Sci/Tech). Using a RoBERTa-Base backbone ($M=20$, $\epsilon=1.0$, $R=5.0$), Table~\ref{tab:agnews} shows that VPDR consistently outperforms the IGPP baseline across all Dirichlet skew levels. Notably, under extreme label skew ($\alpha=0.1$), VPDR improves average accuracy by nearly 3\%. This confirms that VPDR's noise redistribution and norm regularization effectively preserve semantic representations in dense natural language embeddings just as robustly as they do in visual features.
\begin{table}[hbpt]
\centering
\caption{Average test accuracy (\%) of FedProto on AG News.}
\label{tab:agnews}
\vspace{-2mm}
\setlength{\tabcolsep}{2pt}
\fontsize{8}{7.66}\selectfont
\renewcommand{\arraystretch}{1.25}
\begin{tabularx}{\columnwidth}{l||Y|Y|Y|Y|Y} 
\noalign{\hrule height 0.8pt}
\rowcolor{gray!25}
Methods & $\alpha=0.1$ & $\alpha=0.5$ & $\alpha=1.0$ & $\alpha=5.0$ & $\alpha=10.0$\\
\hline\hline
+ IGPP & 34.27& 72.56& 85.72& 89.36&  90.06\\ 
+ VPDR & \textbf{37.00} & \textbf{73.75} & \textbf{86.51} &  \textbf{89.73}  &  \textbf{90.17} \\ 
\noalign{\hrule height 0.8pt}
\end{tabularx}
\end{table}

\subsection{More Details for Section~\ref{sec:attack}}\label{app:Details of Privacy Attacks} 

\begin{table*}[t] 
\centering 
\caption{\textbf{Privacy Attack Results} of FedProto on Digits under varying privacy budgets.} 
\label{tab:digits-attacks} 
\vspace{-2mm} 
\setlength{\tabcolsep}{3pt} 
\fontsize{8}{7.66}\selectfont 
\renewcommand{\arraystretch}{1.25} 
\begin{tabularx}{\textwidth}{l|l||YYY|YYYY} 
\noalign{\hrule height 0.8pt} 
\rowcolor{gray!25} 
& & \multicolumn{3}{c|}{\textbf{Feature Reconstruction (FSH)}} & \multicolumn{4}{c}{\textbf{Membership Inference (MIA)}}\\ 
\cline{3-5}\cline{6-9}
\rowcolor{gray!25} 
\multirow{-2}{*}{$\epsilon$} & \multirow{-2}{*}{LDP} & Cos Sim\scriptsize{$\downarrow$} & cFFD\scriptsize{$\uparrow$} & Top-1 Hit(\%)\scriptsize{$\downarrow$} & ROC-AUC\scriptsize{$\rightarrow0.5$} & TPR@1\%FPR\scriptsize{$\downarrow$} & Advantage\scriptsize{$\rightarrow0$} & F1 Score\scriptsize{$\downarrow$} \\ 
\hline\hline 
- & NoLDP & 0.9999$\pm$0.0000 & 304.8$\pm$36.5 & 100$\pm$0.00 & 0.6094$\pm$0.0224 & 0.0611$\pm$0.0723 & 0.2037$\pm$0.0358 & 0.5844$\pm$0.0733 \\ 
\hline \multirow{3}{*}{1} 
& +IGPP & 0.8253$\pm$0.0018 & 989.2$\pm$28.0 & 83.33$\pm$1.54 & 0.5531$\pm$0.0188 & 0.0288$\pm$0.0381 & 0.1515$\pm$0.0263 & 0.5444$\pm$0.0820 \\ 
& \textbf{+VPP} & 0.8285$\pm$0.0026 & 933.8$\pm$30.2 & 83.00$\pm$2.05 & 0.5539$\pm$0.0166 & 0.0270$\pm$0.0552 & 0.1539$\pm$0.0246 & 0.5440$\pm$0.0856 \\ 
& \cellcolor{vpdr}\textbf{+VPDR}
& \cellcolor{vpdr}0.8289$\pm$0.0020
& \cellcolor{vpdr}940.9$\pm$27.3
& \cellcolor{vpdr}83.23$\pm$2.37
& \cellcolor{vpdr}0.5460$\pm$0.0282
& \cellcolor{vpdr}0.0257$\pm$0.0177
& \cellcolor{vpdr}0.1586$\pm$0.0213
& \cellcolor{vpdr}0.5447$\pm$0.0766 \\  
\hline \multirow{3}{*}{2} 
& +IGPP & 0.8722$\pm$0.0025 & 764.7$\pm$24.0 & 97.12$\pm$0.92 & 0.5775$\pm$0.0194 & 0.0397$\pm$0.0573 & 0.1788$\pm$0.0275 & 0.5586$\pm$0.0845 \\ 
& \textbf{+VPP} & 0.8736$\pm$0.0027 & 765.7$\pm$25.8 & 97.50$\pm$0.00 & 0.5754$\pm$0.0165 & 0.0383$\pm$0.0549 & 0.1761$\pm$0.0240 & 0.5578$\pm$0.0816 \\ 
& \cellcolor{vpdr}\textbf{+VPDR}
& \cellcolor{vpdr}0.8717$\pm$0.0029
& \cellcolor{vpdr}777.1$\pm$15.2
& \cellcolor{vpdr}97.50$\pm$0.00
& \cellcolor{vpdr}0.5727$\pm$0.0320
& \cellcolor{vpdr}0.0390$\pm$0.0557
& \cellcolor{vpdr}0.1723$\pm$0.0234
& \cellcolor{vpdr}0.5534$\pm$0.0755 \\ 
\noalign{\hrule height 0.8pt} 
\end{tabularx} 
\end{table*}

\begin{table*}[t]
\centering
\caption{\textbf{Privacy Attack Results} of FedProto on PACS under varying privacy budgets.}
\label{tab:pacs_attacks}
\vspace{-2mm} 
\setlength{\tabcolsep}{3pt} 
\fontsize{8}{7.66}\selectfont 
\renewcommand{\arraystretch}{1.25} 
\begin{tabularx}{\textwidth}{l|l||YYY|YYYY} 
\noalign{\hrule height 0.8pt} 
\rowcolor{gray!25} 
& & \multicolumn{3}{c|}{\textbf{Feature Reconstruction (FSH)}} & \multicolumn{4}{c}{\textbf{Membership Inference (MIA)}}\\ 
\cline{3-5}\cline{6-9}
\rowcolor{gray!25} 
\multirow{-2}{*}{$\epsilon$} & \multirow{-2}{*}{LDP} & Cos Sim\scriptsize{$\downarrow$} & cFFD\scriptsize{$\uparrow$} & Top-1 Hit(\%)\scriptsize{$\downarrow$} & ROC-AUC\scriptsize{$\rightarrow0.5$} & TPR@1\%FPR\scriptsize{$\downarrow$} & Advantage\scriptsize{$\rightarrow0$} & F1 Score\scriptsize{$\downarrow$} \\
\hline\hline
- & NoLDP   & 0.9999$\pm$0.0000 & 334.3$\pm$47.7 & 100.00$\pm$0.00 & 0.5642$\pm$0.0110 & 0.0187$\pm$0.0037 & 0.2570$\pm$0.0125 & 0.7258$\pm$0.0043 \\
\hline
\multirow{3}{*}{1} & +IGPP   & 0.7271$\pm$0.0031 & 1620.7$\pm$14.4 & 54.21$\pm$2.81 & 0.5009$\pm$0.0108 & 0.0089$\pm$0.0024 & 0.1097$\pm$0.0128 & 0.7083$\pm$0.0028 \\
& \textbf{+VPP}   & 0.7230$\pm$0.0028 & 1612.2$\pm$19.4 & 56.43$\pm$1.87 & 0.5029$\pm$0.0098 & 0.0094$\pm$0.0031 & 0.1101$\pm$0.0126 & 0.7067$\pm$0.0014 \\
& \cellcolor{vpdr}\textbf{+VPDR}
& \cellcolor{vpdr}0.7219$\pm$0.0028
& \cellcolor{vpdr}1626.2$\pm$0.0125
& \cellcolor{vpdr}55.71$\pm$2.14
& \cellcolor{vpdr}0.5004$\pm$0.0083
& \cellcolor{vpdr}0.0102$\pm$0.0024
& \cellcolor{vpdr}0.1066$\pm$0.0096
& \cellcolor{vpdr}0.7073$\pm$0.0014 \\
\hline
\multirow{3}{*}{2} & +IGPP   & 0.7811$\pm$0.0030 & 1237.8$\pm$11.4 & 65.89$\pm$2.95 & 0.5038$\pm$0.0093 & 0.0087$\pm$0.0032 & 0.1135$\pm$0.0113 & 0.7088$\pm$0.0027 \\
& \textbf{+VPP}   & 0.7874$\pm$0.0028 & 1216.8$\pm$23.5 & 65.57$\pm$1.87 & 0.5075$\pm$0.0107 & 0.0092$\pm$0.0031 & 0.1156$\pm$0.0127 & 0.7085$\pm$0.0030 \\
& \cellcolor{vpdr}\textbf{+VPDR}
& \cellcolor{vpdr}0.7862$\pm$0.0029
& \cellcolor{vpdr}1242.1$\pm$17.9
& \cellcolor{vpdr}65.93$\pm$2.04
& \cellcolor{vpdr}0.5026$\pm$0.0080
& \cellcolor{vpdr}0.0101$\pm$0.0022
& \cellcolor{vpdr}0.1076$\pm$0.0103
& \cellcolor{vpdr}0.7076$\pm$0.0015 \\
\noalign{\hrule height 0.8pt}
\end{tabularx}
\end{table*}

We provide detailed descriptions of the attack protocols and metrics used in the privacy risk evaluation in Section~\ref{sec:attack}.  

\subsubsection{Feature-Space Hijacking (FSH)}
FSH aims to reconstruct an input whose encoder feature matches a target class prototype.
Given a target prototype $\mathbf p^*$ and an encoder path $\phi(\cdot)$ that mirrors the client’s prototype pathway, we minimize
\begin{align*}
\mathcal{L} = \big\|\phi(\mathbf x)-\mathbf p^*\big\|_2^2
+\lambda_{\mathrm{TV}} \mathrm{TV}(\mathbf x),
\end{align*}
where the first term minimizes the feature distance between the reconstructed input and the prototype, and the second term penalizes total variation (TV) to promote smoothness in the reconstructed image.

\noindent\textbf{Metrics for FSH} are as follows.
\begin{itemize}[itemsep=1pt, topsep=2pt, leftmargin=10pt]
\item \textbf{Cosine similarity}: Measures the similarity between the reconstructed feature and the target prototype. Lower values in \([-1, 1]\) indicate weaker alignment and hence stronger privacy.
\item \textbf{Classifier Feature Fr\'echet Distance (cFFD)}: Fr\'echet distance between the distribution of reconstructed features and the real class feature distribution. Larger cFFD indicates greater mismatch and thus better privacy.
\item \textbf{Top-1 Hit (\%)}: Percentage of reconstructed images whose nearest class prototype matches the target class. Lower Top-1 hit means the attacker fails to recover the correct class, reducing privacy risk.
\end{itemize} 

\vspace{0.5mm}
\noindent\textbf{Implementation.}
We optimize an unconstrained parameterization with Adam (learning rate $0.01$, up to $10{,}000$ steps), mapping to $[0,1]$ via $\tanh$. We set $\lambda_{\mathrm{TV}} = 10^{-3}$. For practical federated settings, the attack supports both single-target and batched modes, with early stopping (patience $300$, tolerance $10^{-6}$). Input shapes are inferred from a validation loader, and the feature path $\phi(\cdot)$ reuses the clients’ prototype-extraction branch. We attack up to $10$ classes with one prototype per class and a batch size of $16$.

\subsubsection{Membership Inference Attack (MIA)}
MIA is a passive attack where the adversary aims to determine whether a given sample was part of the model’s training set, under the assumption that training samples lie closer to their class prototypes than non-members.
Let $\phi(\cdot)$ be the encoder path used to form prototypes. For a sample $\mathbf{x}$ and privatized prototype $\tilde{\mathbf p}_m^c$, we define the membership score
\begin{align*}
s(\mathbf x;\tilde{\mathbf p}_m^c)=-\big\|\phi(\mathbf x)-\tilde{\mathbf p}_m^c\big\|_2^2. 
\end{align*}
We compute this score for each sample and compare it against a decision threshold to classify the sample as either a member or not. By sweeping the threshold over all possible values of $s(\cdot)$, we obtain an ROC curve, and summarize attack performance via the area under the curve (AUC).

\noindent\textbf{Metrics for MIA} are as follows.
\begin{itemize}[itemsep=1pt, topsep=2pt, leftmargin=10pt]
\item \textbf{ROC-AUC}: Measures the attacker’s ability to distinguish training-set members from non-members. Values closer to $0.5$ indicate chance-level discrimination and thus stronger privacy. 
\item \textbf{TPR@1\%FPR}: The true positive rate (TPR) at a fixed false positive rate (FPR) of 1\%. Lower TPR values indicate that the attacker is unable to reliably distinguish training samples from non-members at a low FPR.
\item \textbf{Advantage}: Defined as $\text{TPR} - \text{FPR}$ at a chosen threshold, measuring the net gain over random guessing. Values near zero indicate near-chance performance. Unlike AUC, which aggregates performance across all thresholds, this metric focuses on a single threshold. 
\item \textbf{F1 score}: Harmonic mean of precision and recall, computed at the optimal threshold. A lower F1 score indicates poor attack accuracy in distinguishing members from non-members, corresponding to stronger privacy.
\end{itemize}

\vspace{0.5mm}
\noindent\textbf{Implementation.}
The maximum number of samples per class in training and testing is set to 800, and the decision threshold is dynamically calculated by maximizing the F1 score on the precision-recall curve.

\subsubsection{More Results on Digits and PACS}
Tables~\ref{tab:digits-attacks} and~\ref{tab:pacs_attacks} report additional privacy attack results on Digits and PACS under the same threat model and metrics as in Section~\ref{sec:attack}. 
Across both datasets and $\epsilon \in \{1,2\}$, all three LDP mechanisms (IGPP, VPP, and VPDR) consistently reduce FSH success and drive MIA towards near-chance performance compared to the non-private baseline. 
The differences among IGPP, VPP, and VPDR are within one standard deviation and show no systematic trend, confirming that VPP and VPDR attain privacy protection comparable to IGPP under the same LDP budgets.

\subsection{Effect of $d$ and $r$}
\label{app:effect-dr}
Varying the embedding dimension $d$ and the budget split ratio $r$ changes how the total privacy budget is allocated between subspace selection and prototype release, which may affect utility and empirical attack performance. Nevertheless, after calibrating the noise according to our analysis, the overall $(\epsilon,\delta)$--LDP guarantee remains the same.

On PACS (Tab.~\ref{tab:dr_check}), the round-averaged MIA-AUC stays close to chance level ($\approx 0.5$) across all tested $d$ and $r$, indicating no observable increase in membership leakage.
In contrast, accuracy exhibits mild variation with a weak interaction between $d$ and $r$: larger dimensions ($d\ge 512$) slightly prefer a moderate split ($r\approx 0.1$), whereas $d=256$ is more tolerant to larger $r$ (e.g., $r=0.2$). Across all settings, VPDR consistently improves accuracy over IGPP while maintaining comparable (near-chance) AUC.

\begin{table}[hbpt]
\centering
\caption{Average test accuracy and round-averaged MIA-AUC.}
\label{tab:dr_check}
\vspace{-2mm}
\setlength{\tabcolsep}{0.01pt} 
\fontsize{8}{7.5}\selectfont
\renewcommand{\arraystretch}{1.25} 
\begin{tabularx}{\columnwidth}{c||YY|YY|YY|YY}
\noalign{\hrule height 0.8pt}
\rowcolor{gray!25} 
\cline{4-9}
\rowcolor{gray!25} 
\multicolumn{1}{c||}{} & \multicolumn{2}{c|}{IGPP} & \multicolumn{2}{c|}{VPDR ($r{=}0.05$)} & \multicolumn{2}{c|}{VPDR ($r{=}0.1$)} & \multicolumn{2}{c}{VPDR ($r{=}0.2$)} \\
\cline{2-9}
\rowcolor{gray!25} 
\multirow{-2}{*}{$d$} & AVG & AUC & AVG & AUC & AVG & AUC & AVG & AUC \\
\hline\hline
$256$  & 90.80\% & 0.5072 & 92.08\% & 0.5041 & 92.59\% & 0.5051 & \textbf{92.82\%} & 0.5029 \\
$512$ & 90.58\% & 0.5009 & 92.26\% & 0.5011 & \textbf{92.71\%} & 0.5004 & 92.67\% &  0.5015 \\
$1024$ & 91.13\% & 0.5019 & 92.54\% & 0.4966 & \textbf{92.67\%} & 0.5003  & 92.28\% & 0.5011 \\  
\noalign{\hrule height 0.8pt}
\end{tabularx}
\end{table}

\end{document}